\documentclass{article}
\usepackage[utf8]{inputenc}
\usepackage{fullpage}
\usepackage{amsfonts,amsmath,amstext,amssymb,amsthm,amstext,color,xcolor}
\usepackage{csquotes}
\usepackage{bbm}
\usepackage{xspace}
\usepackage[colorlinks]{hyperref}
\usepackage{xfrac}
\usepackage{faktor}
\usepackage{pbox}
\hypersetup{
    colorlinks = true,
    linkcolor = blue,
    anchorcolor = blue,
    citecolor = blue,
    filecolor = blue,
    urlcolor = blue
    }
\usepackage{cleveref}
\usepackage{comment}
\usepackage{parskip}
\usepackage{graphicx}
\usepackage{soul}
\usepackage{endnotes}

\theoremstyle{definition}
\newtheorem{theorem}{Theorem}[section]
\newtheorem{definition}[theorem]{Definition}
\newtheorem{remark}[theorem]{Remark}

\newtheorem{proposition}[theorem]{Proposition}

\newtheorem{claim}[theorem]{Claim}

\newcommand{\C}{\mathcal{C}}

\renewcommand{\S}{\mathcal{S}}

\newcommand{\R}{\mathbb{R}}

\newcommand{\y}{\mathbf{y}}

\usepackage[utf8]{inputenc} 
\usepackage[T1]{fontenc}    
\usepackage{hyperref}       
\usepackage{url}            
\usepackage{booktabs}       
\usepackage{amsfonts,amsmath}       
\usepackage{nicefrac}       
\usepackage{microtype}      
\usepackage{xcolor}         
\usepackage{booktabs}
\usepackage{colortbl}

\let\mc\multicolumn
\usepackage{graphicx,subcaption}

\definecolor{Gray}{gray}{0.90}
\definecolor{LightCyan}{rgb}{0.88,1,1}

\newcolumntype{a}{>{\columncolor{Gray}}c}
\newcolumntype{b}{>{\columncolor{white}}c}

\title{Towards Visual Foundation Models of Physical Scenes }

\author{Chethan Parameshwara${}^*$ \and Alessandro Achille\thanks{Equal contribution} \and Matthew Trager \and Xiaolong Li \and Jiawei Mo \and Jianbo Ye \and Ashwin Swaminathan \and C.J. Taylor \and Dheera Venkatraman \and Xiaohan Fei${}^*$ \and Stefano Soatto${}^*$ \and \texttt{\{xiaohfei,soattos,aachille,cmparam\}@amazon.com} }
\date{
AWS AI Labs \\
May 17, 2023}

\begin{document}

\maketitle

\begin{abstract}
We describe a first step towards learning general-purpose visual representations of physical scenes using only image prediction as a training criterion. To do so, we first define  ``physical scene'' and  show that, even though different agents may maintain different representations of the same scene, there is a notion of physical scene that can be uniquely defined and inferred. Then, we show that NeRFs cannot represent the physical scene, as they lack extrapolation mechanisms. Those, however, could be provided by Diffusion Models, at least in theory. To test this hypothesis empirically, NeRFs can be combined with Diffusion Models, a process we refer to as NeRF Diffusion, used as unsupervised representations of the physical scene. Our analysis is limited to visual data, without external grounding mechanisms that can be provided by independent sensory modalities.  
\end{abstract}

\section{Introduction}

Vision serves to infer properties of {\em the scene} from images. But in reality each agent, whether natural or artificial, can only perceive the scene through a finite set of observations, which are compatible with infinitely many scenes. It is not clear, then, what {\em ``the scene''} even means, and how scenes inferred by different agents for different purposes may relate. For example, for the purpose of navigating the surrounding environment, the scene may be fruitfully described as a geometric configuration of surfaces, whereas for the purpose of recognizing edible objects, photometric statistics aggregated locally may be more useful. In general, different tasks may require inferring geometric (shape), photometric (reflectance, illumination), dynamic (motion), functional (affordances) or semantic (identities and relations) properties of the scene.

It seems implausible, therefore, that models trained to simply predict images, such as Neural Radiance Fields (NeRFs) and Diffusion Models (DMs), oblivious of the complexities of light and matter, may ``capture the scene'' in the sense of inferring a representation that can subtend all visual tasks. However, data prediction is a universal task and, in other domains such as language modeling, simply predicting unseen tokens appears sufficient to learn a representation that can support a multitude of downstream tasks and seemingly engender high-level behavior such as reasoning (``chain-of-thought'') and transduction (``in-context learning''). At first glance, the language realm and physical environment appear rather different: Language originates in the human brain which is not accessible, whereas the physical environment can be probed with multiple independent modalities. Yet, we argue, the ``true scene'' is a chimera, and there may be more similarities between ``language models'' and ``world models'' than meet the eye. 
    
In this paper, we explore whether and how NeRFs and/or Diffusion Models {\em can} learn a representation of {\em the scene}. The outline of the paper and its contributions are summarized next.

\subsubsection*{High-level Summary}

First, we (i) define the scene in Sect.~\ref{sec:scene}. We argue that, from the point of view of an agent that makes observations, the scene can only be meaningfully defined as an {\em abstract concept}, with no underlying objective entity. Nonetheless, (ii)  if a ``true'' or ``objective'' scene exists which generates the measurements, we show that it fits our definition (Theorem~\ref{thm:real}). Unfortunately, (iii) this is not useful since,  given two valid scenes compatible with the measurements, an agent cannot decide which of the two, if any, is the ``real'' one (Theorem~\ref{thm:impossible}). Nonetheless, we can (vi) define a notion of ``physical scene''  that is unique and therefore unambiguous (Theorem~\ref{thm:uniqueness-scene}).

The results in Section \ref{sec:prelim} suggests that an abstract model, such as that implemented by a Deep Neural Network, can in principle represent the physical scene. NeRFs are an obvious first candidate, since ostensibly they are designed and trained to infer the radiance function, whose structure depends on the geometry and photometry of the scene. Unfortunately, (vii) in Section~\ref{sec:towards} we show that a NeRF, at least in its basic implementations, {\em cannot} represent a scene (Proposition \ref{prop:nerf}). However, Diffusion Models provide what is missing: In Proposition~\ref{prop:dm} we show that (viii) the composition of a NeRF with a Diffusion Model, which we refer to as NeRF Diffusion, can be a viable representation of the physical scene.\endnote{At this stage, we emphasize that NeRFs and DMs are only one of many possible candidates, and the arguments that follow hold for other architectures, for instance Transformers. We also emphasize that the definition of ``physical scene'' that follows is also compatible with infinite collections of images, although in practice at any given time we only have a finite set.} 

The argument summarized in (i)-(viii) represents the basis on which we design our first implementation of a model that, potentially, can represent physical scenes. The conditions under which a model can represent a scene require that the model can not only interpolate, but also {\em extrapolate} details that are not manifest in a dataset, doing so in a manner that is compatible with the aggregate statistics of all seen scenes. To empirically validate the claim that a NeRF augmented with a Diffusion Model is a valid representation, we therefore need to show that it can ``hallucinate'' unseen details in a way that is perceptually,\footnote{This is the model's perception, not human perception: The criterion is for synthesized and real images to be phenomenologically indistinguishable by the model itself.} if not objectively, accurate. 

To this end,  in Section~\ref{sect:implementation} we describe (ix) a novel diffusion-based method to learn the structural priors and sampling operators so that, composed with a NeRF, a Diffusion Model can extrapolate realistic images. Instead of training diffusion models from scratch, we fine-tune pre-trained Stable Diffusion with NeRF rendering to shape the priors. Through learned priors, the diffusion model enhances the rendering on novel view synthesis to the point where it could be used to extrapolate details indefinitely.

To the best of our knowledge, we are first to attempt a derivation of the characteristics that a trained model should have in order to represent ``the physical scene'' without any objectivity requirement, using data alone. In this sense, little prior work in the literature has tackled (i)-(viii) outside the language domain \cite{pietroski2018conjoining}. There is, on the other hand, a massive and rapidly growing literature on NeRFs and Diffusion Models, as well as on their combination (ix), which we point to throughout the paper and describe in a Sec.~\ref{sec:related} in order to not disrupt the flow. In Sec.~\ref{sec:discussion} we discuss the limitations of our approach and potential avenues for future work.

\section{Background: Representations of Visual Scenes}
\label{sec:prelim}

We think of a camera as a function that yields measurements that depend on its own configuration as well as on a latent entity,  or ``scene.'' As anticipated, defining the scene as an objective entity is problematic, but we posit that, whatever it may be, it is shared among all agents immersed in it. Since each agent observes different views of it, or even if they all observe the same views, each may process them differently or impose different priors, the representations of the shared scene are different in each agent; yet, all these versions are {\em related} in the sense of being representations of the same underlying scene. Since no version is special in any sense, what defines the scene, then, is none of them, but rather their relation, as we articulate in the next subsections. 

\subsection{What is ``a scene''? Scenes as abstract concepts}
\label{sec:scene}

Let $y=(y_1, \ldots, y_k)\in \mathbb R^k$ be a random variable representing the measurements coming from $k$ different sensors. Associated to these measurements, there is a configuration $c \in \C$ of the sensors ({\em e.g.}, the pose of the camera). To simplify the notation, we write $\y = (y, c)$. An agent observes samples from the joint distribution $p(y, c)$.\footnote{It could be argued that agents do not observe the configuration $c$ of their own sensors but either control it, or have to infer it from the data. However, taking this point of view requires introducing a notion of temporal continuity which complicates the theory. Simply assuming that $c$ is known or already inferred from data allows us to develop most of the theory without having to deal with time, and without loss of generality since whatever part of $c$ cannot be inferred from $y$ does not affect the representation anyway.} 
After several measurements are obtained, $(\y_1, \ldots, \y_n)$, the agent may observe that they are correlated. For example, moving around a table we may observe that the \textit{perceived} brown quadrilateral maintains an overall similar color and its shape changes predictably with the viewpoint $c$ \cite{russell}. 
These correlations could be ``explained'' by an underlying factor, which we call {\em a scene}: 
\begin{definition}[A Scene]
\label{defn:scene}
A scene is any random variable $S \in \S$ that renders the measurements independent, that is 
\[
p(\y_1, \ldots, \y_n) = \int p(\y_1|S) \ldots p(\y_n|S) dP(S).
\]

In particular, no pair of measurements share any information $I(\y_1;\y_2|S) = 0$, once a scene is known. In this sense, a scene ``explains'' the measurements.\footnote{The reader may be wondering why we are integrating with respect to the scene, when a NeRF is inferred from measurements of a single scene. Unfortunately, the language of statistics does not allow us to deal with individual entities, but one can imagine the agent being dropped into random scenes, each time executing a number of measurement, and repeating the process a number of times without knowledge of whether the underlying scene is the same, indefinitely.} 
We call \textit{representation of a scene} the pair  $\mathbf{r} = (p(\y|S), P(S))$ of probability distributions that define the scene.
\end{definition}

\begin{remark}[Representation and hallucination]
\label{rmk:scene-representation}
We wish to emphasize the difference between a representation of a scene and a representation of the given images. For a finite set of observations, $\y_{\leq t} = (\y_1, \ldots, \y_t)$, a representation is simply any sufficient statistic, including the data themselves.  On the other hand, a representation of the scene comprises two elements: $p(\y|S)$, which can be used to  generate observations from the scene, and $P(S)$ which can be seen as a \textit{prior} over the possible scenes. A non-trivial consequence of this choice is that a representation of the scene can be used to \textit{hallucinate}, or predict realistic measurements, something that a representation of the measurements does not. To see this, consider the problem of hallucinating a new measurement $\y \sim p(\y|\y_{\leq t})$ conditioned on past observations. This can be written as:
\[
p(\y|\y_{\leq t}) = \frac{p(\y, \y_{\leq t})}{p(\y_{\leq t})} = \int p(\y|S) \frac{ p(\y_1|S) \ldots p(\y_t|S)}{\int p(\y_1|S) \ldots p(\y_t|S) dP(S)}  dP(S) = \int p(\y|S) dP(S|\y_{\leq t}).
\]
From the second equality, we see that knowing the scene representation $\mathbf{r} = (p(\y|S), P(S))$ gives us all the information we need to hallucinate new realistic observations. 
\end{remark}

Based on this definition, ``the scene'' is not unique, which is why we refer to it as ``a scene'' instead. Later we will show that, if {\em ``the''} scene exists, in the sense commonly referred to as the ``real'' or  ``true'' scene, then it is a scene in the sense of the definition. But, while the agent cannot know whether such a ``true'' scene actually exists, a scene always does, under mild assumptions: 

\begin{theorem}[Existence of a scene]
Suppose the measurements are exchangeable, meaning that 
\[
p(\y_1, \ldots, \y_n) = p(\y_{\pi(1)}, \ldots, \y_{\pi(n)}),
\]
where $\pi: [n] \to [n]$ is any permutation of $[n] = \{1, \ldots, n\}$ and that each finite sequence can be extended to an infinite exchangeable sequence of measurements. Then, there exists a random variable $S$ such that
\[
p(\y_1, \ldots, \y_n) = \int p(\y_1|S) \ldots p(\y_n|S) dP(S),
\]
hence $S$ is a scene.
\end{theorem}
This follows directly from  De Finetti's theorem  \cite{diaconis1980finite}. Now that we have (at least) a scene, we define a measurement function, or {\em presentation} \cite{koenderink2011vision} {\em not} of the real scene, but of any scene. 

\begin{definition}[Measurement function and ``presentation'']
Let $S \in \S$ be a scene. The induced measurement function $h$ is a stochastic function $h: \S \times \C \to \R^k$ defined by $h(S, c) = y$ where $y \sim p(y|S, c)$.  We call the function $p(y | \cdot, c)$ a {\em presentation}, which is instantiated for any given scene as $p(y|S, c)$. 
\end{definition}
Note that specifying the presentation function, as a computational procedure, is equivalent to specifying a scene, although we do not refer to as a {\em representation} of the scene, for reasons clarified in Remark \ref{sec:presentations}. The fact that specifying a scene and its presentation are equivalent should make it clear that a scene, as defined, is not some ``objective,'' ``true,'' ``material,'' or ``physical reality,'' but an entirely abstract entity that may be embodied in a number of ways, for instance as a digital computer or as a neural network. By sampling from $p(\cdot | S, c)$, a scene can produce infinitely many ``controlled hallucinations'' \cite{koenderink2011vision}, as discussed in detail in Remark \ref{sec:presentations}. 
While not objective, a scene is connected to reality in two ways: In one direction, if the real scene exists, then it is a scene, as we show next. In the other direction, if all scenes can be related to each other in some canonical way, then we can define such a canonical scene as a proxy of the ``real'' scene without any ontological complications. We do so in the next section.

\begin{theorem}[The ``real scene'', if it exists, is a scene]
\label{thm:real}
Assume that there is an objective entity  $S$ that generates the measurement of the agent thorough a measurement function $h$. That is, assume that $y = h(S, c)$. Then $S$ is a scene.
\end{theorem}
This follows directly from the definition but, despite being seemingly innocuous, the statement is problematic since any passive observer, whether natural or artificial, has only access to the given measurements, which are never enough to capture the ``real'' scene unless we make strong assumptions about it (for instance finiteness), which would inevitably be unverifiable \cite{achille2022binding}. This is captured by the next statement.

\begin{claim}[A passive agent cannot know whether the scene is real]
\label{thm:impossible}
Let $S$ and $S'$ be two scenes. Suppose that the sequence of measurements $\y_1,\ldots,\y_n$ is generated by first flipping a fair coin $z \sim \operatorname{Bernoulli}(\frac{1}{2})$ to decide whether to use $S$ or $S'$, and then sampling measurements from that scene using its respective measurement function. Then the agent cannot infer $z$ beyond chance level.
\end{claim}
This is a simple consequence of the fact that the agent only observes samples from $p(\y_1, \ldots, \y_n)$ and both scenes have the same marginal distribution. The upshot is that, if there are two scenes that generate identical measurements, an agent cannot tell the difference, so the notion of {\em the ``true scene''} is a chimera.
Moreover, two scenes $S$ and $S'$ may be incompatible, meaning that there may be no computable functions $f$ or $g$ such that $S = f(S')$ or $S' = g(S)$. Hence, even if there was a real scene underlying the data, each individual agent may infer its own version,  which bears no relation to that of another agent, resulting in a perceptual Babel. The question, then, is whether there is some notion of scene that all agents who share the same observations can agree to, so they can at least share information about it. We call this the {\em physical scene}. We preface that, by necessity, the physical scene is still an abstract concept, dependent on the measurements available to each agent.

Since individual scenes are not directly comparable, we therefore seek some sort of minimality criterion that each agent can enforce in order to arrive at some sort of ``canonical'' scene which, if unique in some sense, we can refer to as the {\em physical scene}.

\subsection{How are scenes related? ``The physical scene''} 

In Computer Vision we tend to think of ``the scene'' as the underlying ``cause'' of a given collection of images. Unfortunately, as we saw in the previous section, this {\em naive realism} is fallacious\footnote{Quoting  (pp. 14-15): {\em ``We all start from `naive realism', i.e., the doctrine that things are what they seem. [...] The observer, when he seems to himself to be observing a stone, is really, if physics is  to be believed, observing the effects of the stone upon himself. Thus science seems to be at war with itself: when it most means to be objective, it finds itself plunged into subjectivity against its will. Naive realism leads to physics, and physics, if true,  shows that naive realism is false. Therefore naive realism, if true, is false; therefore it is false.''}}~\cite{einstein-russell}. 
Instead, we flip the script and think of the given images, which are objective, as underlying an infinite set of possible scenes, and the key question is how all these scenes are {\em related.}

\begin{definition}[Sufficient statistics of a scene]
A sufficient statistic of $S$ for the measurement is any function $T(S)$ such that
\[
p(\y|S) = p(\y|T(S)),
\]
or equivalently such that we have the Markov chain $S \to T(S) \to y$. Note that this is a sufficient statistic \textit{of the scene} to generate the measurement, not a sufficient statistic of the measurement to infer the scene.\footnote{This flipping of the focus from the data to the scene corresponds to a change in perspective from {\em naive realism} to a {\em analytical philosophy} or, using Koenderink's nomenclature \cite{koenderink2011vision},  from the ``Marrian'' to the ``Goethean'' accounts of Vision.} Equivalently, $T(S)$ is a subset of the information in the scene that is sufficient to explain the measurement $I(\y; S) = I(\y; T(S))$.
\end{definition}
Note that thus far we have said nothing about how a sufficient scene $S$ arises. All we have said is that, if an entity exists that satisfies the definition of a scene, it can generate all the measurements that the real scene, if it existed, would generate. Later we will address identifiability and observability. For now, of all sufficient statistics, we are interested in the simplest possible ones, which are generally not unique.

\begin{definition}[Minimal sufficient statistics]
A minimal sufficient statistic $T(S)$ of $S$ is a sufficient statistic such that, given any other sufficient statistic $T'(S)$ we have $T(S) = f(T'(S))$ for some function $f$.
\end{definition}

We now have a candidate for a notion of scene that agents who observe the same measurements could share. Next, we will prove that such a scene is unique in a strong sense, making it canonical, which allows us to refer to it as {\em ``the (physical) scene''}. 

\begin{theorem}[Existence of a minimal sufficient scene]
Given a sufficient scene $S$, we can always construct a minimal sufficient scene $S_m = T(S)$ under the weak condition that the support of $p(S|\y)$ is the same for all $\y$.\footnote{This is the case, for instance, if the measurements are affected by Gaussian noise, whose support covers the domain. } In particular, the function $L: \S \times \mathcal{M} \to \R$ defined by
\[
L(S, y) = \frac{p(S|\y)}{p(S|\y_0)} \propto \frac{p(\y|S)}{p(\y_0|S)}
\]
for any $\y_0 \in \mathcal{M}$ is a minimal sufficient scene.
\end{theorem}
The following result justifies naming a minimal sufficient scene {\em the physical scene}. 

\begin{theorem}[Strong uniqueness and the physical scene]
\label{thm:uniqueness-scene}
Let $S$ and $S'$ be two sufficient scenes and let $S_y$ and $S'_y$ be two minimal sufficient scenes corresponding to $S$ and $S'$ respectively. Then there are functions $f$ and $g$ such that $S_y = f(S'_y)$ and $S'_y = g(S_y)$.
\end{theorem}
We therefore call the unique minimal sufficient scene for a set of given measurements the {\em physical scene} subtending those measurements. Note that the physical scene has to be compatible with physical laws that involve measured quantities, for such laws could be interpreted as production rules for the measurements. Both physical laws and physical scenes are abstract concepts, in the sense that they can be represented in the memory of the agent, even though we have not yet described how such a scene could be inferred from data.

The existence theorem is a straightforward application of the standard existence theorem for minimal sufficient statistics \cite{keener2010theoretical,casella2021statistical}. The uniqueness theorem, however, is stronger, as normally we would conclude that all minimal sufficient statistics of the same scene are in a bijection, but could not compare between different scenes. To get the stronger version, note that if $S$ and $S'$ are  scenes, then the tuple $(S, S')$ is also a scene and both $S_y$ and $S'_y$ are minimal sufficient statistics of it. It then follows that they are in a bijection. Since a bijection is an equivalence relation, we can also represent the physical scene as an equivalence class, drawing a parallel to concepts in natural language, as we describe next.

\begin{remark}[Scenes as equivalence classes] As we anticipated, any given collection of images is compatible with infinitely many scenes, none of which is ``special'' (canonical), so the physical scene cannot be any one of them, but rather their {\em relation}, which is a bijection, hence an {\em equivalence relation}. Specifically, let $S$ be a scene  and $p(y |S, c)$ its {\em presentation}. Then, presentations define an equivalence relation among scenes, whereby $S' \sim S \Leftrightarrow p(y|S,c) = p(y|S', c)  \ \forall \ c$, and 
corresponding equivalence classes $[S] = \{S' \ | \ S' \sim S\}$. Then, we have that, if $S$ is a scene, $[S]$ is a physical scene. 
Note that, if we had  defined relations {\em not among scenes} based on the images they hallucinate, {\em but among measured images} as done in classical  objectivist ``Marrian'' style of Computer Vision, we would not have equivalence classes, because sets of images from the same scene are related by co-visibility, which is not a transitive relation: The fact that the frustra (pre-images) of two images $y_1, y_2$ intersect, and so do the frustra of $y_2, y_3$, does not imply that that frustra of $y_1$ and $y_3$ intersect. 
\end{remark}
The previous remark can be summarized in the following claim: 
\begin{claim}[Physical Scenes as Equivalence Classes]
if $S$ is a scene, then $[S] = \{S' \ | \ p(\cdot | S', c) = p(\cdot | S, c) \ \forall c \}$ is the corresponding physical scene.
\label{claim:equivalence}
\end{claim}
Note that changing the measurement function may change the corresponding physical scene, much as laws of physics may need to be amended if new instruments produce evidence that invalidates the laws.

\begin{remark}[Presentations as ``controlled hallucinations''] 
\label{sec:presentations}
Note that scenes induce equivalence classes of {\em infinite collections} of images, not of {\em different finite collections of given images of a scene}. Rather, they are {\em the distributions of images that two scenes could ``hallucinate''} using their presentation function. For this reason, even if $p(\cdot |S, c)$ defines the equivalence class that represents the scene, we refer to it as ``presentation'' rather ``representation,'' following Koenderink \cite{koenderink2011vision}. A ``representation'' presumes that something is present to begin with, which can be manipulated and thence re-presented. But, unlike the data, the scene $S$ is an abstract entity which is not accessible to be  manipulated or re-presented. Nonetheless, it can be used to {\em present} images that are compatible with the given ones, a process  referred to as ``controlled hallucinations.'' The input $c$ provides the control, and the presentation function is the abstract mechanism that the scene uses to generate images, or more properly hallucinate them since the process is an abstraction of data formation, not an actual measurement. 
\end{remark}
The interpretation in the previous remark may provide hints to the reader of what characteristics a neural network should possess to be a Foundation Model of physical scenes, which we develop in the next section, after a few remarks.

\begin{remark}[``Large Word Models'' vs. Large Language Models]
One could call a map from images to scenes $\y \mapsto [S]$ a  ``large\footnote{Although we have given no indication that such a model should be ``large,'' the functional anatomy of primate neocortex \cite{felleman1991distributed} suggests that a large portion of real estate in human brains is devoted to processing visual information, so if language models are large, so should be visual models.} world model'' (LWM), much in the same way in which a large language model (LLM) is a map from sentences to meanings \cite{soatto2023taming}. Just as the physical scene is an equivalence class of images that can be inferred by a LWM, which can then be used to hallucinate infinitely many images, a ``meaning'' in its most elementary and naive sense, can be defined an equivalence class of sentences that can be inferred by a LLM, which can then be used to hallucinate infinitely many sentences controlled by the given ``prompt'' \cite{soatto2023taming}. Since both sentences and images can be embedded (see \Cref{rmk:scene-representation}) as vectors in a metric space by their corresponding models, then the goal of a Foundation Model, whether operating on image or language or other data, is to give meaning to the data, which it can do by mapping data to a metric space and constructing equivalence classes therein, as well as to model the ``prior'' distribution $p([S])$ of realistic scenes and the generative distribution $p(\y|[S])$. For sensory data, the meaning rests in the scene that generates them, and is provided by {\em grounding}; for language data, the meaning can only be defined by using human-provided annotations, a challenge discussed next.
\end{remark}

\begin{remark}[Grounding]
Just like a trained LLM attributes meaning to sentences by constructing equivalence classes of them, so a LWM attributes meaning to images, in the sense above, by constructing equivalence classes of scenes. In general, meaning is not {\em intrinsic} in data,  but rather {\em attributed} to data, a process that requires an external entity \cite{pietroski2018conjoining}. This entity could be the one that {\em originates} the measured data, for instance the surrounding environment, informally referred to as the ``true'' or ``real'' scene, for the case of sensory data. Relating measured data to the source is called {\em grounding}. But unlike natural language data, which originates in the human brain that is not accessible for experimentation, the surrounding environment can be probed with independent sensors by an embodied entity. As pointed out in \cite{soatto2023taming}, environmental grounding is not veridical but it is falsifiable: We cannot verify the existence of an object, say ``chair,'' in an image, for there are no chairs in the images, just pixels, but an embodied agent can attribute meaning to a chair by testing its ability to sit on it (affordance). On the other hand, grounding in language {\em must} rely on induction -- which is not falsifiable -- based on human annotations -- which are subjective and therefore not veridical despite the language being a closed model: If training data have inconsistent human annotations, the ``ground truth'' used to build the discriminant that defines  meanings is inconsistent, and so is the resulting system of meanings. Not so for the surrounding environment, so long as it exists and can be probed with independent modalities.
\end{remark}
This fundamental difference between LLMs and LWMs, which is the possibility of cross-modal grounding, is highlighted but not exploited in this paper, which only considers world models that can be inferred from passively gathered, remote, non-contact, distributed sensory data, in particular images.

\begin{remark}[Self-supervision, Contrastive Learning, and the universality of the prediction task]
\label{sec:prediction}
In order to be viable, a scene must generate images through its presentation function that, while generally different from the ones that the ``true scene'' would generate if we could measure them, are {\em  indistinguishable} (distributionally equivalent, see \Cref{rmk:scene-representation}.) For images that we {\em did} measure, this defining characteristic of scenes indicates a natural learning criterion, which is simply {\em data prediction}, or ``masked'' reconstruction. That is, we use the presentation function to hallucinate images from a vantage point we {\em did} observe, and then compare the hallucinated images with the ones we {\em actually} measured. The masking can be causal and delayed, where prediction of future images is compared with actual images to be taken a few steps in the future, once they are observed. Prediction is universal in the sense that, if a model is capable of predicting future images, it is therefore sufficient to support any function of future images instantiated at inference time. The question, then, is whether it is minimal. In the absence of any information about the task, the only representation that is sufficient is the data itself or any lossless representation of it, for the downstream task may be exact retrieval. It is common to refer to so-called ``self-supervised learning,'' including contrastive learning, as ``task agnostic.'' This is incorrect since the task is implicit in the choice of nuisance variability that defines the surrogate losses, data augmentations, or transformations that are manually designed. For example, contrastive learning that imposes that rotated versions of the same image represent the same equivalence class cannot, trivially, be useful for the task of determining image rotation. Similarly, masked autoencoding is a way of simulating occlusion, which is useful if there are occlusion nuisances, but not if the sensory modality is X-ray tomography of magnetic resonance where there are no occlusions. On the other hand, all variability in visual data is manifest over time, for an embodied agent, which makes prediction a natural choice. 
\end{remark}

\section{Towards Foundational Models of Physical Scenes} 
\label{sec:towards}

The previous section established the scene as an abstract concept with different embodiments, which could be a computer program or, say, a Deep Neural Network (DNN). This raises hopes that NeRFs may be a representation of the scene and therefore  a viable tool in developing generic models of the world from images for the purpose of any downstream task. In other words, we ask whether NeRFs and/or Diffusion Models could be general Foundational Models for physical scenes, in the sense of implementing a {\em presentation function}.

\subsection{NeRFs as (re)presentations of scenes?}
\label{sec:nerf-not}

We will describe NeRFs in more detail in Sec.\ref{sect:implementation}, but for now it suffices to say that they are maps $h$ learned from images of a scene $S$ with knowledge of pose $c$ that, at inference time, can be used to map a novel pose $c'$ onto an image $y = h(S, c')$. But do they implement a presentation function? The fact that, in a NeRF, the map $h$ is implemented as a feed-forward memoryless operator using a multi-layer perceptron (MLP), is sufficient to tamper expectations. 
\begin{proposition}NeRFs cannot represent scenes.
\label{prop:nerf}
\end{proposition}
The proof follows from Proposition 2.9 in \cite{achille2022binding}.
By \Cref{rmk:scene-representation} a representation of the scene needs to be able to hallucinate new \textit{realistic} measurements given a set of past measurements. On the surface, NeRFs can use past data to create an occupancy model, and then use this model to generate images from novel viewpoints. However, those images need not be realistic.
This is obvious in practice and illustrated in Fig. \ref{fig:supp-nerf-artifacts}. A NeRF of an object  captured rotating around it at a certain distance can be used to generate images from, say, twice the distance but any attempt to ``zoom in'' will generate artifacts that reveal the NeRF $\tilde S$ as not {\em the real (presentation function of the) scene}. In other words, the attempt to extrapolate is {\em unrealistic}. The meaning of {\em unrealistic} here is not some vague notion of perceptual similarity having to do with the human visual system, but  the precise conditions of Theorem \ref{thm:uniqueness-scene}: Images captured by the camera are obviously distinguishable from those generated by the NeRF, so $\exists \ c \ | \ p(\cdot | {\tt GoPro}, c) \neq p(\cdot | {\tt NeRF}, c)$, violating the conditions of the theorem.

At the core of the problem is the fact that while NeRFs provide a measurement model $p(y|S)$, they do not provide a distribution $P(S)$ over the scenes, which encodes  ``realism.'' That is, wherever the given data provides constraints for interpolation, the model should faithfully reproduce it, which it does since it is trained to do so explicitly through the reconstruction error. However, where the data do not provide sufficient constraints, the model should extrapolate so that, if data were to become available at different granularity, it could be compared even if it was not trained on it. Comparison is in a distributional sense, so a viewer or model would not be able to tell which image comes from the ``true scene,'' whatever that is, and which comes from the NeRF. 
For a trained model, this rests on induction, which is why we need to train a model on {\em different scenes}, and then use the result to hallucinate details where data from the extant scene is insufficient. In doing so, however, the model must remain faithful to the conditions of Theorem \ref{thm:uniqueness-scene}. To do so, we bring in Diffusion Models.

\subsection{Pushing a NeRF towards the physical scene with Diffusion Models}

We will use a Diffusion Model (DM) as a map from the output of a NeRF, $\tilde y = h(S, c)$ to an image $y$, which is in the form of a conditional distribution $p(y|h(S, c))$ from which images can be sampled. We defer the details to Sec.~\ref{sect:implementation}, but for now what matters is the fact that DMs
are trained to approximate the distribution of natural images, conditioned on a ``prompt,'' and given the scale-invariance properties of low-level statistics~\cite{zoran2009scale}, they can be used to hallucinate details where absent, with the hallucination driven by the inductive bias of the trained model. Indeed, DMs have been used to ``denoise'' or ``super-resolve'' images, which  are ill-posed inverse problems where the inductive bias is used to hallucinate details and structure absent in the data. Clearly a standard DM itself is not a viable presentation function. However, we hypothesize that, when composed with (conditioned on) the output of a NeRF, they are viable presentation functions.

We call the composition of the NeRF $h$ and the diffusion model {\em NeRF Diffusion}, and we indicate the corresponding measurement function as $f_S$, which is a stochastic function. The following proposition shows that NeRF Diffusion models {\em are} a viable presentation function, and therefore equivalent to {\em the physical scene} given the images used for training. 

\begin{proposition}[NeRF Diffusions as (re)presentations of the physical scene]
\label{prop:dm}
Given a set of observations $\y_{\leq t} = (\y_1, \ldots, \y_t)$, a NeRF Diffusion defines a stochastic $f_{\y_{\leq t}}: \C \to \R^D$ such that $f_{\y_{\leq t}} \sim \int p(y|S, c) p(S| \y_t) dS$. This is a representation of the scene where we take the scene $S$ to be $S=(\y_1, \ldots, \y_n)$.
\end{proposition}
\begin{proof}
We need to show that $f_{\y_{\leq t}}$ is enough to compute the distribution $P(S) = P(\y_1, \ldots, \y_n)$. To this end, note that
\[
P(\y_1, \ldots, \y_n) = \prod_{t=1}^n P(\y_t|\y_{<t}) = \prod_{t=1}^n  f_{\y_{< t}}(\y_t).
\]
\end{proof}
$f_S$ is also minimal as a function, although any particular implementation may be super-minimal due to implementational inefficiencies.  In the next section, we describe an implementation, which we consider a first step towards a proper model of the physical scene, exploiting methods from the current state of the art. 

\section{NeRF Diffusion implementation}
\label{sect:implementation}

In this section we describe a method to adapt a Diffusion Model to complement a NeRF so as to implement a presentation function that satisfies the conditions of Theorem~\ref{thm:uniqueness-scene}.  

\textit{Neural Radiance Fields (NeRFs)}
~\cite{mildenhall2021nerf} are multi-layer perceptrons trained to approximate the radiance field, \textit{a.k.a.} Plenoptic Function. To render a pixel, points are sampled along the ray passing through the pixel and originated from the camera with configuration $c$. Density and view-dependent radiance values are obtained by querying the MLP with the points' spatial coordinates and viewing direction. Volumetric rendering is then performed on the density and radiance values to produce the color at the pixel. NeRF approaches~\cite{mildenhall2021nerf, barron2021mip, barron2022mip, verbin2022ref, mildenhall2022nerf} are mostly trained with calibrated images ($c$ is known during training as part of the camera intrinsic and extrinsic calibration) supervised by pixel-wise discrepancy. Camera pose and calibration could also be inferred from the images as part of a pre-processing stage, for instance using standard Structure From Motion tools.

\textit{Diffusion Models (DMs)} are statistical models of the distribution of images obtained by learning the reverse diffusion operator of a process that maps training images to noise~\cite{nelson1967dynamical,lindquist1979stochastic,lindquist1979stochastic}. DMs can be conditioned on a number of inputs, from text to sketches, low-resolution images or, in our case, unrealistic NeRF renderings. In other words, we condition the DMs on images produced by a NeRF, so it learns to tilt the generative distribution implemented by a NeRF, which is unrealistic, to one that is realistic, hence compatible with the physical scene.

To realize a \textit{NeRF Diffusion}, we need to train a DM to denoise views synthesized by a NeRF and extrapolate details not manifest therein. The outcome   is improved  phenomenological quality (realism) of the rendering, achieved by learning NeRF priors (Fig.~\ref{fig:block_diagram}). 
We train a DM to learn the distribution of NeRF rendering artifacts and map the resulting aliased renderings to the corresponding calibrated images. To that end, we use Nerfacto within Nerfstudio \cite{nerfstudio}: We turn off pose refinement and train Nerfacto for 40K iterations in about 30 minutes on the ObjectScans dataset described in Sec.~\ref{sect:experiments}, using a single NVIDIA V100. We then train a DM based on Stable Diffusion~\cite{rombach2021highresolution}, a large-scale text-to-image latent diffusion model. To leverage the benefit of text-to-image generation capabilities \cite{wang2022pretraining}, which provides a strong prior to hallucinate semantically plausible details, we initialize the weights of our model with a pretrained Stable Diffusion checkpoint. To condition the DM, we directly add the NeRF rendering's latent features into Stable Diffusion's U-net structure, as in \cite{zhang2023adding}. 

Given a latent image $y_0$ and its corresponding NeRF rendering's latent feature $S$, diffusion algorithms progressively add noise to the image to produces a noisy version $y_t$. Here $t$ is the number of steps in which noise is added. Since we are learning the NeRF priors to model the artifacts, we set the text prompt $u_t$ of Stable Diffusion to generic prompt (i.e. ``high resolution and high quality"). The overall objective is
\[
\mathcal{L}= \mathbb{E}_{y_0, t,u_t,S, \epsilon \sim \mathcal{N}(0,1)}\left[ ||\epsilon - \epsilon_\theta(y_t, t, u_t , S)||^2 _2\right].
\]

We train our DM on all 121 sequences from ObjectScans' NeRF for 10K iterations in 6 hours using 8 NVIDIA V100 GPUs. The final trained model captures the NeRF's structural priors and sampling operators so that, when a new condition (rendering) is presented, the model extrapolates realistic images, thus realizing a simple embodiment of a NeRF Diffusion.

\section{Experimental evaluation}
\label{sect:experiments}

\subsection{Details on the datasets} 
\label{sec:datasets}

\label{sect:product_scan_dataset}

{\bf ObjectScans} consists of 121 video sequences of common household objects such as cookware, containers, and toys. The videos were captured by the authors using GoPro Hero9 action cameras at $30$ FPS and $1920 \times 1080$ resolution. The length of each video ranges from 20 seconds to about 1 minute.
Unlike existing datasets that are either captured by stationary DSLR cameras~\cite{barron2022mip,mildenhall2019local} or synthesized~\cite{mildenhall2021nerf}, recording videos makes it possible to scale the data acquisition process for training NeRFs on real-world scenes. Yet, video data pose a challenge to the established NeRF techniques on account of motion blur and compression artifacts not present in existing datasets. To minimize motion blur, we move the camera slowly and steadily around the object during data acquisition. We also observed that NeRFs work best when the camera tightly samples the space of viewpoints. We employed these best practices in our data collection. For each video sequence, we uniformly sample about 300 frames to conduct our experiments totalling about $36K$ frames in the whole dataset. We split the sampled frames into train/test subsets with a $9:1$ ratio, and then use COLMAP~\cite{schonberger2016structure} -- an open-source  Structure-from-Motion software -- to compute camera intrinsics, posesand sparse point clouds from the sampled video frames. 

\label{sect:mipnerf360_dataset}
{\bf Mip-NeRF 360} was created by \cite{barron2022mip} and consists of 9 scenes (5 outdoors and 4 indoors), of which only 7 (namely, \texttt{bicycle, bonsai, counter, garden, kitchen, room, stump}) are publicly available. Each of the scenes contains a central object or area that has relatively complex geometry and/or texture and a background that contains fine-grained details (\textit{e.g.}, foliage). The dataset also comes with camera pose pre-computed using COLMAP. We use these pre-computed camera poses in our experiments.

\subsection{Evaluation metrics}
We evaluate our method on \textit{faithfulness} through similarity metrics and \textit{realism} through distribution metrics. 

\textit{Similarity metrics}:  PSNR, SSIM \cite{ssim_paper}, LPIPS \cite{lpips_paper}. PSNR calculates the ratio of maximum pixel value to noise. PSNR is a sensible performance measure for conpression and transmission, not for representation. SSIM is sensible for perceptual similarities based on luminance, contrast, and structure, but our goal is not to achieve indistinguishability as defined by human perception, but by presentation functions. LPIPS is a criterion based on convolutional features trained for image classification~\cite{krizhevsky2017imagenet}. As such, it is generally more appropriate to evaluate scene representations. 

\textit{Distribution metrics}: IS \cite{inception_score_paper}, FID \cite{fid_paper}, KID \cite{kid_paper}. In FID, a score is calculated by computing the Fréchet distance between two Gaussians fit to features computed by the Inception network. KID measures the dissimilarity between the distributions of real and generated samples without assuming any parametric form, and is also more sample efficient. The inception score is calculated by first using a pre-trained Inception v3 model to predict the class probabilities for each generated image. The Inception score is computed only on generated images without the need for real images. For all distributional metrics, we precompute real statistics using 10k patches of real image samples from each test sequence of our dataset.



 \subsection{Quantitative results}

Table~\ref{table:1} provides the comparison of our approach with Nerfacto~\cite{nerfstudio} -- an efficient and top-performing NeRF model -- and Real-ESRGAN~\cite{wang2021real} -- a state-of-the-art (SOTA) GAN-based image restoration model -- on ObjectScans. We fine-tune both ours and Real-ESRGAN on Nerfacto renderings. Ours has a lower FID, KID score than both Nerfacto and Real-ESRGAN, reflecting the higher plausiblity of the NeRF rendering. We also observe that our approach is competitive with baselines in similarity metrics (\textit{e.g.}, PSNR and LPIPS). Even though these metrics over-penalize high-frequency details (which is commonly observed in diffusion-based image generation), our approach preserves the original details faithfully while removing the artifacts and enhancing realism. Since SSIM evaluates pixel value changes and diffusion-based approaches are designed to hallucinate, our approach under-performs when compared to Nerfacto baseline.

\begin{table}[h!]
\centering
\caption{\small\textit{Quantitative results on ObjectScans.} Ours outperforms the baselines in various distribution metrics (FID, KID, IS) showcasing our model's capability to predict visually plausible images. Furthermore, our model respects measurements (ground-truth images) sampled from the scene registering competitive similarity scores (PSNR, SSIM, LPIPS).}
\resizebox{\textwidth}{!}{%
\begin{tabular}{c rrraaa rrraaa bbbaaa}
\toprule
 & \mc6c{Nerfacto} 
 & \mc6c{Real-ESRGAN}
  & \mc6c{Ours}\\
 \cmidrule{2-7}
 \cmidrule{8-13}
 \cmidrule{14-19}
  Sequences & PSNR$\uparrow$ &SSIM$\uparrow$ &LPIPS$\downarrow$ & FID$\downarrow$ & KID$\downarrow$ & IS $\uparrow$ &
  PSNR$\uparrow$ &SSIM$\uparrow$ &LPIPS$\downarrow$ & FID$\downarrow$ & KID$\downarrow$ & IS $\uparrow$ &
  PSNR$\uparrow$ &SSIM$\uparrow$ &LPIPS$\downarrow$ & FID$\downarrow$ & KID$\downarrow$ & IS $\uparrow$ \\
\midrule
\emph{RiceCooker} & 23.78 & 0.86 & 0.32 & 100.89  & 0.057 & 6.74 & 22.91 & 0.85 & 0.26 & 45.84 & 0.020 & 7.08 & 24.97 & 0.86 & 0.17 & 21.15 & 0.0060 & 7.65\\
\emph{WaterFilter} & 24.17 & 0.87 & 0.26 & 81.34  & 0.0398 & 5.61 & 23.70 & 0.86 & 0.22 & 45.69 & 0.019 & 6.35 & 26.17 & 0.85 & 0.19 & 23.96 & 0.0058 & 6.49 \\
\emph{InstantPot} & 24.11 & 0.85 & 0.21 & 88.75  & 0.046 & 6.24 & 22.93 & 0.81 & 0.17 & 31.99 & 0.012 & 6.69 & 25.88 & 0.83 & 0.11 & 18.15 & 0.007 & 7.05\\ 
\emph{Jar} & 23.19 & 0.84 & 0.28 & 76.78  & 0.034 & 6.25 & 22.34 & 0.81 & 0.21 & 45.70 & 0.020 & 6.87 & 25.38 & 0.83 & 0.15 & 27.25 & 0.01 & 6.56\\
\emph{TheraGun} & 24.13 & 0.87 & 0.26 & 74.15 & 0.035 & 5.32 & 23.69 & 0.86 & 0.22 & 42.17 & 0.019 & 5.77 & 27.5 & 0.87 & 0.15 & 19.16 & 0.0049 & 6.18\\
\emph{Robot} & 24.29 & 0.85 & 0.26 & 68.51 & 0.029 & 6.13 & 23.56 & 0.83 & 0.20 & 44.08 & 0.020 & 6.47 & 27.80 & 0.84 & 0.13 & 23.17 & 0.008 & 6.54\\
\emph{GlassPotLid} & 23.59 & 0.72 & 0.37 & 101.59  & 0.06 & 4.15 & 22.93 & 0.69 & 0.38 & 68.50 & 0.032 & 4.89 & 24.04 & 0.66 & 0.27 & 31.74 & 0.012 & 4.76\\
\emph{LanternOff} & 23.03 & 0.79 & 0.26 & 83.89 & 0.043 & 6.52 & 21.94 & 0.75 & 0.25 & 43.95 & 0.014 & 7.32 & 24.12 & 0.78 & 0.16 & 24.32 & 0.01 & 7.29\\
\emph{LanternOn} & 22.67 & 0.78 & 0.27 & 83.22  & 0.045 & 6.69 & 21.82 & 0.75 & 0.25 & 43.63 & 0.016 & 7.71 & 24.05 & 0.78 & 0.16 & 23.42 & 0.0088 & 7.20\\
\emph{LegoBus} & 24.14 & 0.78 & 0.26 & 49.78  & 0.021 & 6.74 & 22.47 & 0.73 & 0.23 & 26.81 & 0.006 & 6.58 & 24.33 & 0.74 & 0.14 & 17.05 & 0.0061 & 6.33\\
\emph{MacBook} & 21.50 & 0.80 & 0.28 & 72.38 & 0.038 & 5.2 & 21.04 & 0.78 & 0.25 & 37.34 & 0.015 & 5.88 & 24.62 & 0.80 & 0.15 & 23.21 & 0.0092 & 6.10\\
\emph{SquareGlassJar} & 23.68 & 0.80 & 0.25 & 83.19  & 0.043 & 6.16 & 22.70 & 0.74 & 0.25 & 54.26 & 0.019 & 7.39 & 24.82 & 0.78 & 0.15 & 25.89 & 0.0102 & 7.52\\
\emph{StainlessSteelHotpot} & 21.77 & 0.72 & 0.38 & 95.54  & 0.053 & 4.91 & 21.35 & 0.71 & 0.37 & 57.23 & 0.022 & 5.83 & 24.08 & 0.69 & 0.28 & 25.92 & 0.0078 & 5.27\\
\emph{WaterBoiler} & 22.59 & 0.79 & 0.26 & 86.14 & 0.034 & 7.92 & 21.57 & 0.76 & 0.25 & 47.79 & 0.016 & 8.84 & 23.62 & 0.78 & 0.16 & 23.54 & 0.0079 & 7.80\\
\midrule
Average & 23.33 & {\bf 0.81} & 0.28 & 81.87  & 0.041 & 6.04 & 22.50 & 0.78 & 0.25 & 45.36 & 0.018 & {\bf 6.69} & {\bf 25.10} & 0.79 & {\bf 0.17} & {\bf 23.42} & {\bf 0.0081} & 6.62 \\
\bottomrule
\end{tabular}
}

\label{table:1}
\end{table}

Table~\ref{table:2} compares ours against Nerfacto baseline on public Mip-NeRF 360 dataset. Our approach outperforms Nerfacto in distribution metric -- similar to our observations in Table~\ref{table:1}. However, the Nerfacto's similarity metric performs better especially in PSNR and SSIM. The performance gap in similarity metric could be due to fact that we fine-tune the DM which was already trained on ObjectScans and the new fine-tuned model struggles to generalize. To evaluate the generalization hypothesis, we trained the DM only on Mip-NeRF 360 sequences and evaluated the performance on \emph{counter} sequence. We observe performance improvement in our approach (PSNR = 23.79, SSIM = 0.70, LPIPS = 0.25) when compared to baseline (PSNR = 22.90, SSIM = 0.82, LPIPS = 0.37) in terms of similarity metrics (particularly PSNR and LPIPS). We leave the out-of-distribution generalization of our approach for the future work.   

\begin{table}[h!]
\centering
\caption{\small\textit{Quantitative results on Mip-NeRF 360 dataset.}}
\resizebox{0.75\textwidth}{!}{%
\begin{tabular}{c rrraaa rrraaa}
\toprule
 & \mc6c{Nerfacto} 
 & \mc6c{Ours}\\
 \cmidrule{2-7}
 \cmidrule{8-13}
  Dataset & PSNR$\uparrow$ &SSIM$\uparrow$ &LPIPS$\downarrow$ & FID$\downarrow$ & KID$\downarrow$ & IS $\uparrow$
  &PSNR$\uparrow$ &SSIM$\uparrow$ &LPIPS$\downarrow$ & FID$\downarrow$ & KID$\downarrow$ & IS $\uparrow$\\
  
\midrule
\emph{bicycle} & 23.15 & 0.69 & 0.26 & 48.10 & 0.0084 & 4.21 & 21.24 & 0.47 & 0.21 & 26.17 & 0.0027 & 4.03 \\
\emph{bonsai} & 29.77 & 0.92 & 0.07 & 31.21 & 0.0154 & 3.98 & 24.02 & 0.75 & 0.10 & 20.25 & 0.0066 & 4.19\\
\emph{counter} & 22.90 & 0.82 & 0.37 & 102.72 & 0.048 & 7.28 & 20.87 & 0.64 & 0.30 & 23.50 & 0.012 & 6.96 \\
\emph{garden} & 24.62 & 0.70 & 0.38 & 111.56 & 0.0634 & 4.89 & 22.11 & 0.64 & 0.24 & 23.23 & 0.0074 & 4.34\\
\emph{kitchen} & 28.17 & 0.83 & 0.18 & 55.51 & 0.026 & 5.38 & 23.74 & 0.65 & 0.16 & 20.93 & 0.0057 & 5.21\\
\emph{room} & 30.84 & 0.87 & 0.27 & 80.38 & 0.0312 & 5.36 & 25.08 & 0.69 & 0.24 & 42.72 & 0.01 & 7.38\\
\emph{stump} & 25.95 & 0.71 & 0.24 & 65.51 & 0.055 & 2.64 & 22.94 & 0.50 & 0.26 & 25.04 & 0.011 & 3.00 \\
\midrule
Average & \textbf{26.48} & \textbf{0.79} & 0.25 & 70.72 & 0.0353 & 4.81 & 22.85 & 0.62 & \textbf{0.22 }& \textbf{25.97} & \textbf{0.0078} & \textbf{5.01}\\
\bottomrule
\end{tabular}
}
\label{table:2}
\end{table}

\begin{figure}
     \centering
     \begin{subfigure}[b]{0.9\linewidth}
         \centering
         \includegraphics[width=0.9\textwidth]{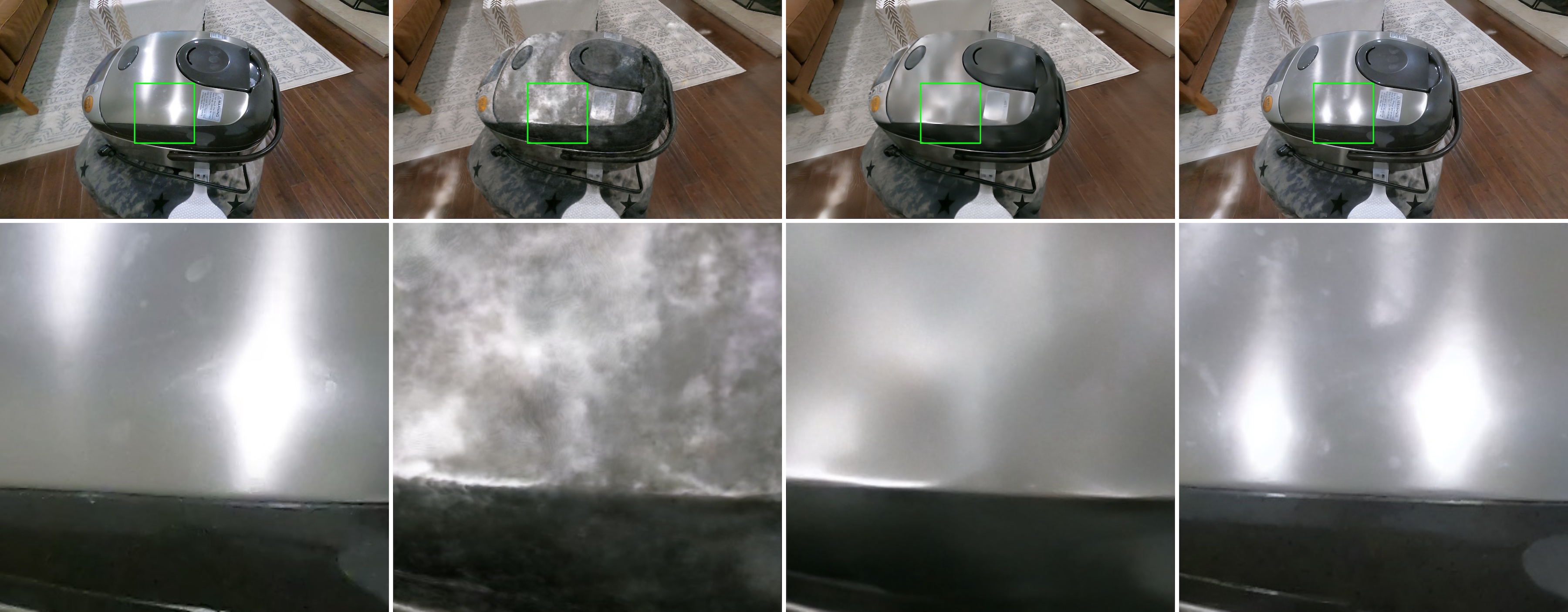}
     \end{subfigure}
     \hfill
     \begin{subfigure}[b]{0.9\linewidth}
         \centering
         \includegraphics[width=0.9\textwidth]{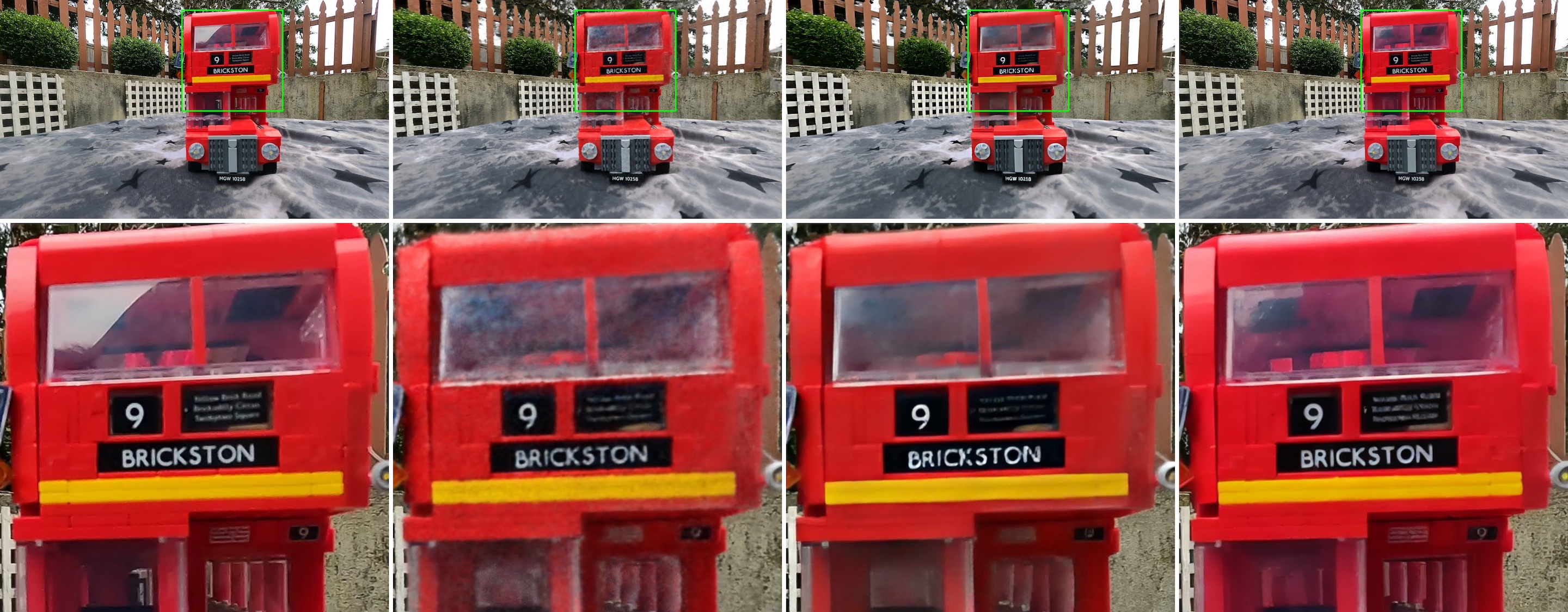}
     \end{subfigure}
     \hfill
     \begin{subfigure}[b]{0.9\linewidth}
         \centering
         \includegraphics[width=0.9\textwidth]{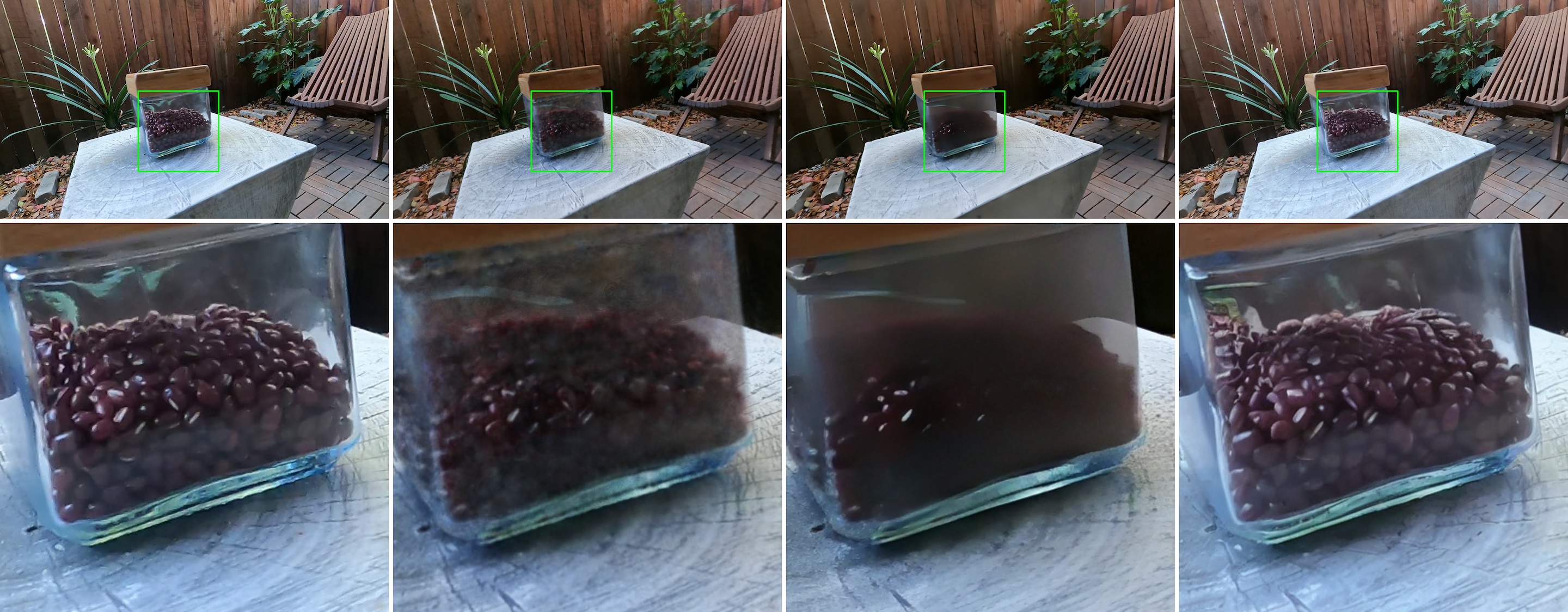}
     \end{subfigure}
     \hfill
     \begin{subfigure}[b]{0.9\linewidth}
         \centering
         \includegraphics[width=0.9\textwidth]{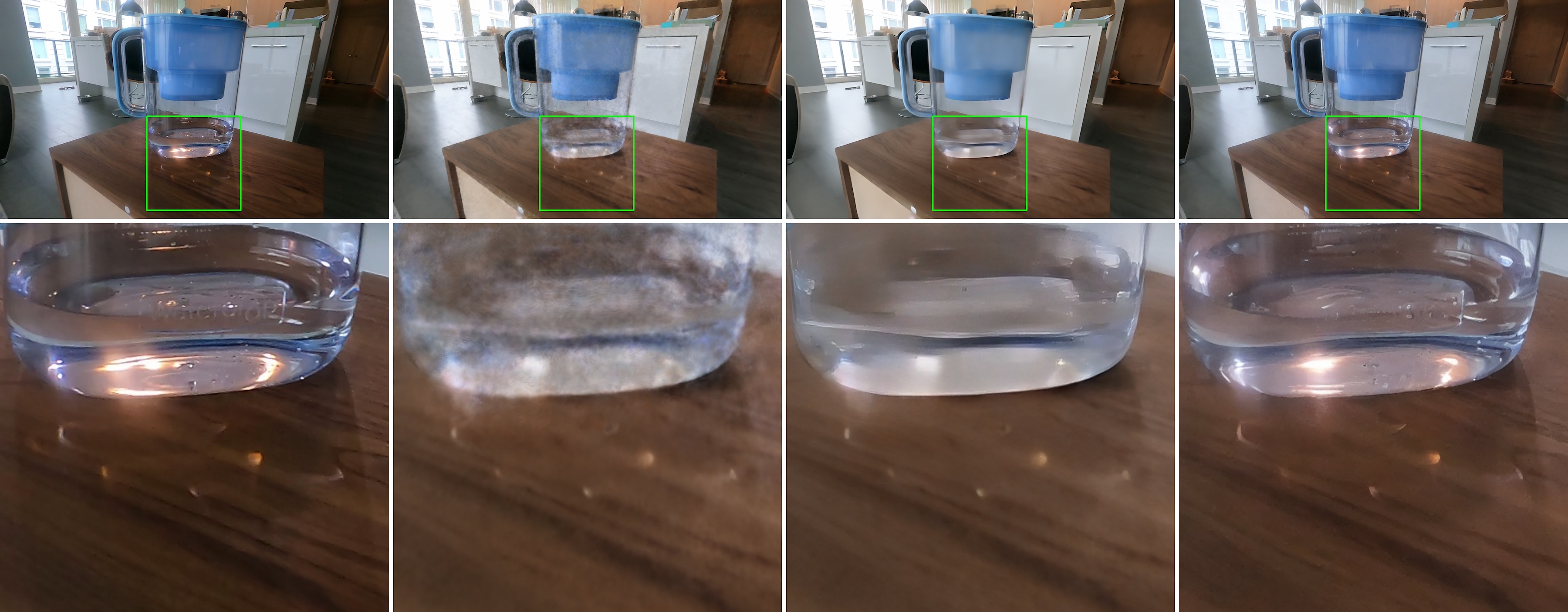}
     \end{subfigure}
        \caption{\small \textit{Visual comparison} of the proposed approach against baselines on our datasets. \textbf{Left to right}: ({\bf a}) ground-truth reference image, ({\bf b}) NeRF rendering, ({\bf c}) NeRF rendering augmented by Real-ESRGAN~\cite{wang2021real} -- SOTA in image restoration, and ({\bf d}) NeRF rendering augmented by our method. \textbf{Top to bottom}: 4 samples with corresponding zoom-in views. Note, the goal is \textit{not} to make the output exactly the same as the ground truth, \textit{but to generate a plausible image of the underlying scene given an imperfect sample (column two)}.}
        
        \label{fig:visual_comparison_object_scan}
        
\end{figure}

\begin{figure}
    \centering
     \begin{subfigure}[b]{0.9\linewidth}
         \centering
         \includegraphics[width=\textwidth]{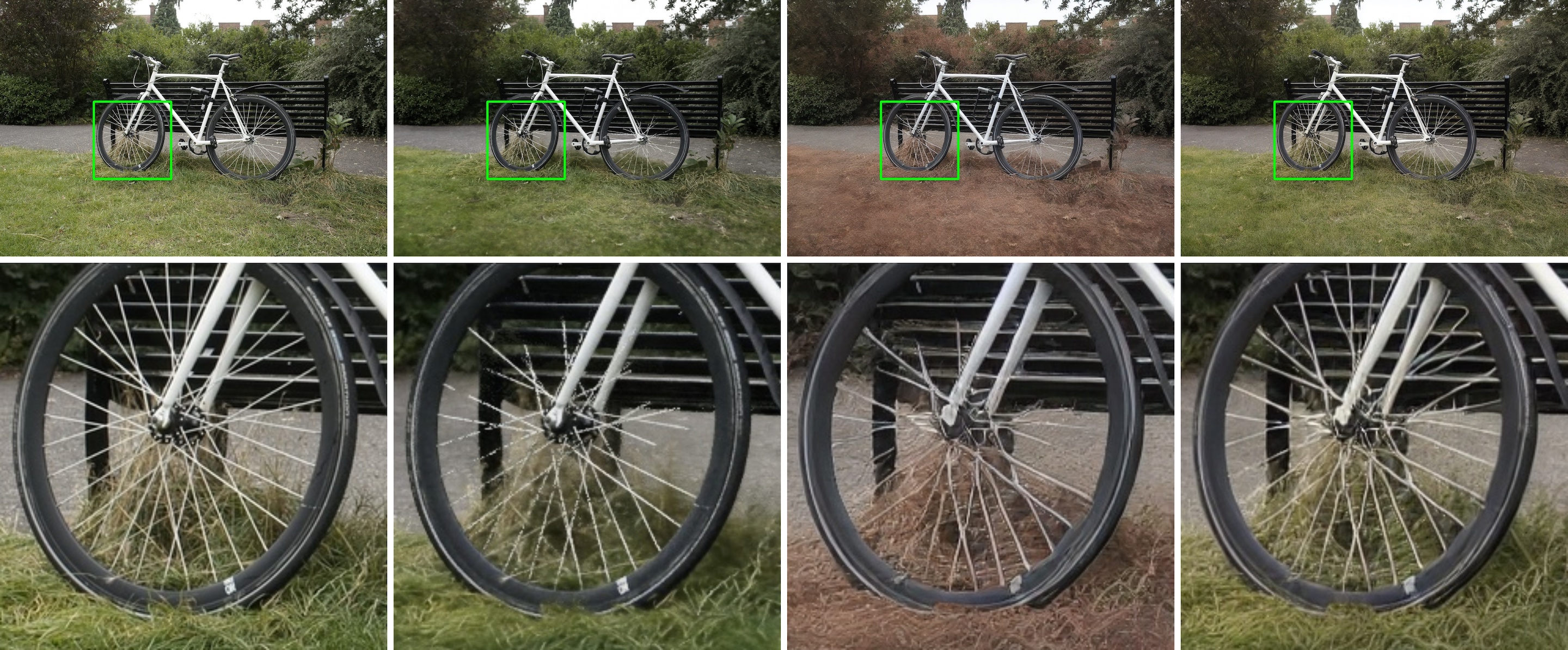}
     \end{subfigure}
     \hfill
     \begin{subfigure}[b]{0.9\linewidth}
         \centering
         \includegraphics[width=\textwidth]{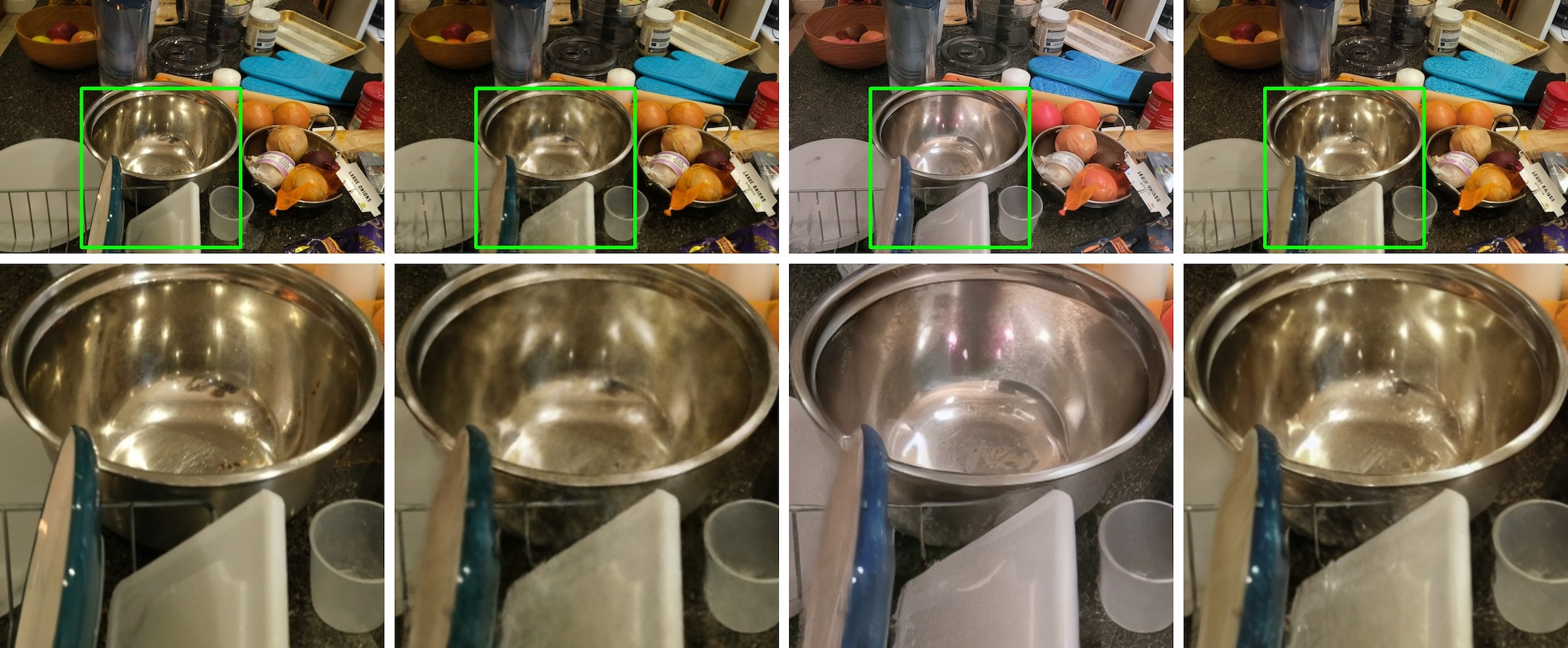}
     \end{subfigure}
    \caption{\small\textit{Qualitative results} of the proposed approach on the public Mip-NeRF 360 dataset. \textbf{Left to right}: ({\bf a}) ground-truth reference image, ({\bf b}) NeRF rendering, ({\bf c}) rendering augmented by diffusion model pre-trained on our ObjectScans dataset, and ({\bf d}) rendering augmented by the diffusion model in ({\bf c}) which is further fine-tuned on Mip-NeRF 360 dataset. Bottom row shows each sample's corresponding zoom-in view. 
   }
    \label{fig:visual_comparison_mipnerf360}
\end{figure}

\subsection{Qualitative results}

Fig.~\ref{fig:visual_comparison_object_scan} presents qualitative results on ObjectScans. NeRFs  to model fine-grained details and glossy surfaces which can be seen in column two. A generic image denoising approach such as Real-ESRGAN fine-tuned on our dataset is able to remove some artifacts but fails to extrapolate structures beyond the Nyquist frequency, resulting in over-smoothing. The proposed approach removes artifacts without oversmoothing. \textit{Furthermore, the proposed approach is able to hallucinate realistic details which may be absent in the scene but nevertheless compatible with the given data and with the natural image statistics}, for example, the reflection on the rice cooker (top rows) and highlights at the bottom of the water filter (bottom rows).

Fig.~\ref{fig:visual_comparison_mipnerf360} presents on qualitative results Mip-NeRF 360. The scenes in Mip-NeRF are out-of-distribution relative to  ObjectScans, and thus the DM pre-trained on the latter is biased to generates samples with inconsistent color distribution. Yet, the generated samples are realistic, which is the goal of the DM, and compatible with the data, which is the goal of the NeRF.  By fine-tuning the DM on Mip-NeRF, we are able to match the color distribution, and also fill in fine-grained details missing in NeRF's renderings (column two). Note the added details such as reflections on the bowl, not necessarily present in the scene but nevertheless realistic.

\subsection{Super-resolution}
Fig. \ref{fig:supp-nerf-artifacts} illustrates the extrapolation characteristics of our NeRF Diffusions. For a given test sample with ground-truth image and camera pose available, we first obtain the NeRF rendering and its corresponding NeRF Diffusion output. Also, we sample additional camera poses near the test pose closer to the object to obtain the ``zoomed-in" version of the NeRF rendering. In the case of Fig. \ref{fig:supp-nerf-artifacts}, we focus on rendering the steel vessel and oil bottle. Since NeRFs fail to interpolate at finer resolution, the NeRF's rendering would have more artifacts. Through feeding this aliased NeRF rendering to a DM, we obtain visually plausible and enhanced scenes through DM's hallucinations/extrapolation. 

\subsection{Limitations}
It is known that diffusion models require considerable amounts of training data.  As can be seen in Fig.~\ref{fig:visual_comparison_mipnerf360} column 3, the sample image generated by our model is biased towards the data distribution that the model has been trained on, since our models are trained on small data. Fine-tuning the model on a specific data domain or training the diffusion model on larger dataset can mitigate the problem.

Another problem that challenges diffusion models is their inability to model fine-grained structure like text,  which we also observed in (Fig.~\ref{fig:limitation_texts}). Other limitations include sampling speed and high computation costs. While ``hallucination'' is a problem in synthetic data generation, in our case it is the goal, so long as the NeRF Diffusion remains faithful to the data where available. The synthesis is tasked to add details that are compatible with natural image statistics, without affecting the rendering where information is manifest in the given data.

Note that, in this aspect, world models are fundamentally different from language models, due to {\em scale variability in the data} which is absent in language. Given a sentence, there are no words revealed as we change the measurement conditions. Given an image, on the other hand, there are always more details about the scene to be revealed if we move closer to objects. These details are not inferrable from the given images by interpolation, which is why a Foundation Model for visual scenes must incorporate a recursion mechanism of other mechanism to go to the limit. 

\begin{figure}
    \centering
    \includegraphics[width=0.8\textwidth]{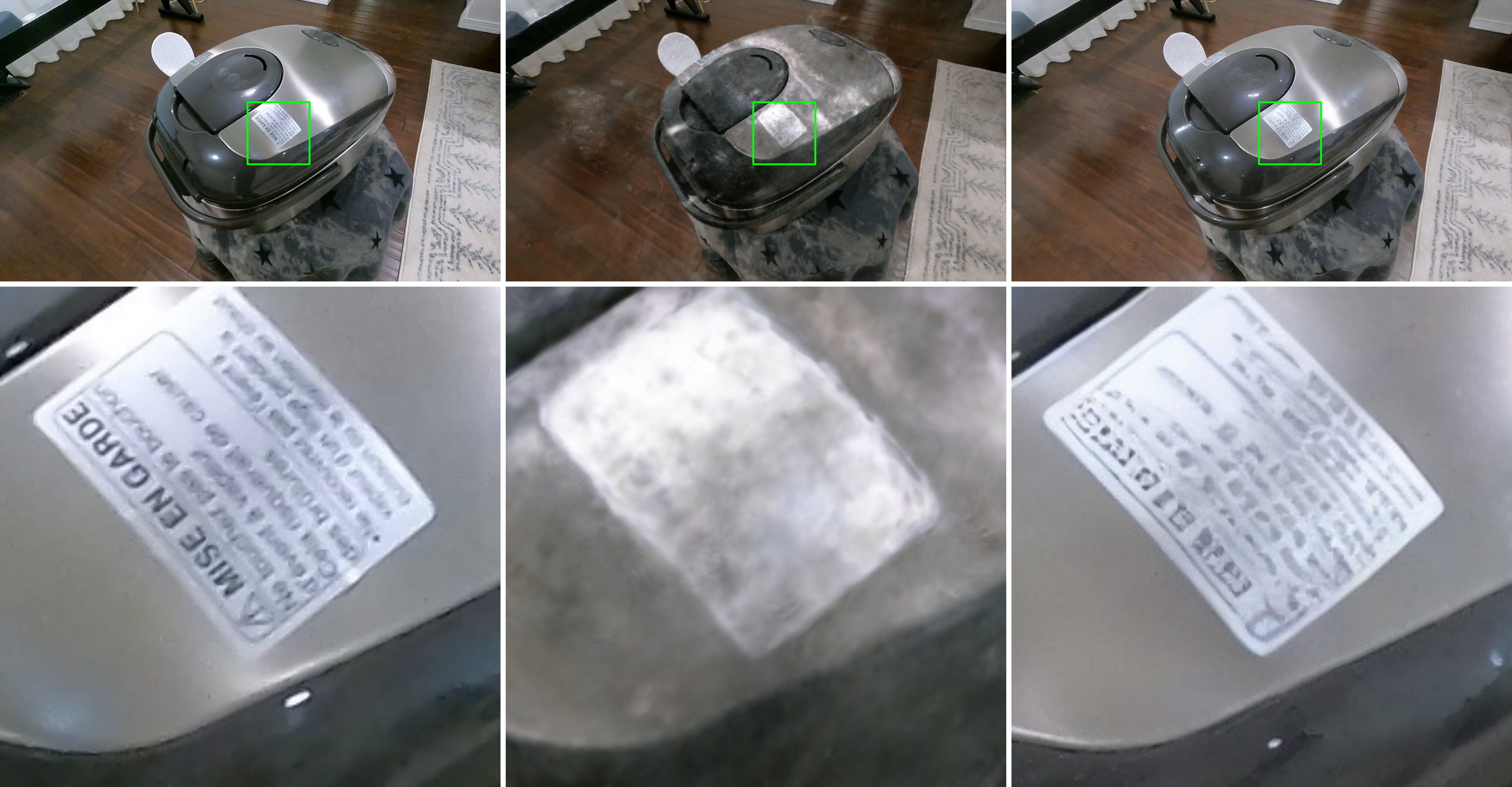}
    \caption{\small\textit{A known challenge} for diffusion models is fine-scale generation, for instance text. \textbf{Left to right}: (a) ground-truth image, (b) NeRF rendering, (c) NeRF rendering augmented by our approach. The bottom row shows zoomed-in views of the top row. Though our approach fills in some high-frequency signals that NeRF failed to capture, it is not capable of hallucinating meaningful text that can be seen in the ground-truth image.}
    \label{fig:limitation_texts}
\end{figure}

\section{Related work}
\label{sec:related}
Diffusion models have significantly advanced image generation offering a scalable and robust training paradigm in either pixel~\cite{ramesh2022hierarchical, saharia2022photorealistic} or latent space~\cite{rombach2022high}, and taking image generation to a new level of realism and diversity. Interest has grown in extending diffusion models to 3D content creation~\cite{wang2022score, xu2022neurallift,  poole2022dreamfusion, zhou2022sparsefusion, lin2022magic3d, melas2023realfusion, karnewar2023holodiffusion}, usually by guiding NeRF training with natural image priors from 2D diffusion models.
DiffusioNeRF~\cite{wynn2023diffusionerf} aims to use a pre-trained denoising diffusion model to regularize NeRF trained on only a few images.
Nerfbusters~\cite{warburg2023nerfbusters} attempts to use diffusion models to remove the the ghosting effects (\textit{a.k.a.} floaters) from NeRF rendering. 

Designing a viable scene representation, or creating a map $y=h(S, c)$ for image prediction has been pursued in different contexts or research fields. Some focus on augmenting existing NeRF models by introducing additional complexity to the map $h$ to (hopefully) circumvent known limitations of the vanilla model such as aliasing or non-Lambertian effects~\cite{mildenhall2021nerf,barron2021mip, barron2022mip, verbin2022ref, mildenhall2022nerf}. Some focus on directly learning an end-to-end denoising or super-resolution network to recover the local statistics of generated images~\cite{ulyanov2018deep,wang2018esrgan}. Those works validate our proposition that a scene is an abstract concept specified via a computational procedure, which has no objective grounding. 

Specifically on NeRFs, 
\label{sect:related-works-nerfs}
\cite{mildenhall2021nerf} and related work~\cite{barron2021mip, barron2022mip, verbin2022ref, mildenhall2022nerf} are state-of-the-art approaches for novel view synthesis. NeRF renders images by encoding density and view-dependent radiance based on multi-layer perceptron (MLP). However, there are several drawbacks. First, training with varying scene resolution leads to aliasing because any 3D scene point is infinitesimally small regardless of its distance to the camera. Mip-NeRF~\cite{barron2021mip} attempts to solve this problem by modeling the 3D scene point using a frustum. The second challenge is with modeling non-Lambertian scenes that include reflection and refraction. Ref-NeRF~\cite{verbin2022ref} addresses reflection modeling by extending the NeRF MLP to predict surface normals and reflection properties (e.g., roughness) in order to calculate reflection direction for specular albedo estimation; this is added to the diffuse color for final rendering. Similarly, NeRFRen addresses the reflection problem by separating transmitted and reflected color. To the best of our knowledge, there is little work on refraction modeling within the volumetric rendering literature. Refraction modeling is addressed based on multi-plane image~\cite{wizadwongsa2021nex} or light field~\cite{suhail2022light}. The next drawback of NeRFs is on their computationally intensive volumetric rendering process. Since we have to query the MLP for every sample point along every camera ray, rendering a single image takes more than several seconds, making it impossible for interactive real-time rendering. Instant-NGP~\cite{muller2022instant} drastically accelerates NeRF rendering by encoding spatial information using hash maps, which reduce the MLP size and enable faster rendering. Nerfstudio~\cite{nerfstudio} is an open-source implementation for various NeRF approaches, including Instant-NGP. 
We have validated that Nerfstudio yields state-of-the-art accuracy within reasonable training time (within half an hour using Tesla V100) and thus used Nerfstudio through out our experiments.

\label{sect:related-works-dataset}
As for datasets, popular choices 
used in the NeRF literature are either synthetic (\textit{e.g.}, NeRF-Synthetic~\cite{mildenhall2021nerf} and BlendedMVS~\cite{yao2020blendedmvs}) or only contain a handful of real-world scenes (\textit{e.g.}, Mip-NeRF 360~\cite{barron2022mip}, LLFF~\cite{mildenhall2019local}). More recent datasets such as MobileBrick~\cite{li2023mobilebrick} and ScanNeRF~\cite{de2023scannerf} contain a few dozen sequences in a laboratory setting: The former captures RGB-D data of LEGO models with known ground-truth CAD models and the latter uses specialized hardware to capture small objects limiting their applicability in real-world scenarios. ScanNet~\cite{dai2017scannet} contains thousands of RGB-D scans of rooms, but its applicability to train NeRFs is yet to be explored. CO3D~\cite{reizenstein2021common} contains about 19K crowd-sourced videos, but focuses on object reconstruction task and as such often has simplistic background. To fill the gap left by existing public datasets, we collected a relatively large dataset using widely available commodity hardware, and conducted some of our experiments using this dataset.

\section{Discussion}
\label{sec:discussion}

As our title suggests, we consider ours only a first step towards building Foundational Models of physical scenes. Nonetheless, to the best of our knowledge, this paper is the first to propose a practical method to represent physical scenes that goes beyond naive realism, dominant since Marr \cite{marr80}, embracing the analytical approach advocated by Koenderink \cite{koenderink2011vision}, formalized and implemented with contemporary Deep Learning tools. 

Evaluating methods that generate details that are not manifest in the given data makes assessment challenging, as all existing quantitative benchmarks fall into the objectivity trap, by  effectively {\em defining} an objective reality where there is none. At the same time, perceptual similarity metrics, which are biased towards the characteristics of the human visual system, are still quite rudimentary, often measuring limited variability (e.g. SSIM) or other coarse distributional statistics (FID, KID, IS). Considerably more work needs to be conducted to develop methods that assess the realism of generated images, as opposed to their fit to ``reality'' defined by an arbitrary benchmark design. We note that our analysis is limited to vision, and must be extended to other modalities, and in particular for active perception by probing the environment with structured signals or contact sensors. We leave these extensions to future work.


\begin{thebibliography}{10}

\bibitem{achille2022binding}
A.~Achille and S.~Soatto.
\newblock On the learnability of physical concepts: Can a neural network
  understand what's real?
\newblock {\em arXiv preprint arXiv:2207.12186}, 2022.

\bibitem{barron2021mip}
J.~T. Barron, B.~Mildenhall, M.~Tancik, P.~Hedman, R.~Martin-Brualla, and P.~P.
  Srinivasan.
\newblock Mip-nerf: A multiscale representation for anti-aliasing neural
  radiance fields.
\newblock In {\em Proceedings of the IEEE/CVF International Conference on
  Computer Vision}, pages 5855--5864, 2021.

\bibitem{barron2022mip}
J.~T. Barron, B.~Mildenhall, D.~Verbin, P.~P. Srinivasan, and P.~Hedman.
\newblock Mip-nerf 360: Unbounded anti-aliased neural radiance fields.
\newblock In {\em Proceedings of the IEEE/CVF Conference on Computer Vision and
  Pattern Recognition}, pages 5470--5479, 2022.

\bibitem{kid_paper}
M.~Bińkowski, D.~J. Sutherland, M.~Arbel, and A.~Gretton.
\newblock Demystifying mmd gans, 2021.

\bibitem{casella2021statistical}
G.~Casella and R.~L. Berger.
\newblock {\em Statistical inference}.
\newblock Cengage Learning, 2021.

\bibitem{cybenko1989approximation}
G.~Cybenko.
\newblock Approximation by superpositions of a sigmoidal function.
\newblock {\em Mathematics of control, signals and systems}, 2(4):303--314,
  1989.

\bibitem{dai2017scannet}
A.~Dai, A.~X. Chang, M.~Savva, M.~Halber, T.~Funkhouser, and M.~Nie{\ss}ner.
\newblock Scannet: Richly-annotated 3d reconstructions of indoor scenes.
\newblock In {\em Proceedings of the IEEE conference on computer vision and
  pattern recognition}, pages 5828--5839, 2017.

\bibitem{de2023scannerf}
L.~De~Luigi, D.~Bolognini, F.~Domeniconi, D.~De~Gregorio, M.~Poggi, and
  L.~Di~Stefano.
\newblock Scannerf: a scalable benchmark for neural radiance fields.
\newblock In {\em Proceedings of the IEEE/CVF Winter Conference on Applications
  of Computer Vision}, pages 816--825, 2023.

\bibitem{diaconis1980finite}
P.~Diaconis and D.~Freedman.
\newblock Finite exchangeable sequences.
\newblock {\em The Annals of Probability}, pages 745--764, 1980.

\bibitem{einstein-russell}
A.~Einstein.
\newblock {\em The philosophy of Bertrand Russell, Part II: Descriptive and
  Critical Essays on the Philosophy of Bertrand Russell}.
\newblock 1946.

\bibitem{felleman1991distributed}
D.~J. Felleman and D.~C. Van~Essen.
\newblock Distributed hierarchical processing in the primate cerebral cortex.
\newblock {\em Cerebral cortex (New York, NY: 1991)}, 1(1):1--47, 1991.

\bibitem{fid_paper}
M.~Heusel, H.~Ramsauer, T.~Unterthiner, B.~Nessler, and S.~Hochreiter.
\newblock Gans trained by a two time-scale update rule converge to a local nash
  equilibrium, 2017.

\bibitem{karnewar2023holodiffusion}
A.~Karnewar, A.~Vedaldi, D.~Novotny, and N.~Mitra.
\newblock Holodiffusion: Training a 3d diffusion model using 2d images.
\newblock {\em arXiv preprint arXiv:2303.16509}, 2023.

\bibitem{keener2010theoretical}
R.~W. Keener.
\newblock {\em Theoretical statistics: Topics for a core course}.
\newblock Springer, 2010.

\bibitem{koenderink2011vision}
J.~J. Koenderink.
\newblock Vision and information.
\newblock {\em Perception beyond inference: The information content of visual
  processes}, pages 27--58, 2011.

\bibitem{krizhevsky2017imagenet}
A.~Krizhevsky, I.~Sutskever, and G.~E. Hinton.
\newblock Imagenet classification with deep convolutional neural networks.
\newblock {\em Communications of the ACM}, 60(6):84--90, 2017.

\bibitem{li2023mobilebrick}
K.~Li, J.-W. Bian, R.~Castle, P.~H. Torr, and V.~A. Prisacariu.
\newblock Mobilebrick: Building lego for 3d reconstruction on mobile devices.
\newblock {\em arXiv preprint arXiv:2303.01932}, 2023.

\bibitem{lin2022magic3d}
C.-H. Lin, J.~Gao, L.~Tang, T.~Takikawa, X.~Zeng, X.~Huang, K.~Kreis,
  S.~Fidler, M.-Y. Liu, and T.-Y. Lin.
\newblock Magic3d: High-resolution text-to-3d content creation.
\newblock {\em arXiv preprint arXiv:2211.10440}, 2022.

\bibitem{lindquist1979stochastic}
A.~Lindquist and G.~Picci.
\newblock On the stochastic realization problem.
\newblock {\em SIAM Journal on Control and Optimization}, 17(3):365--389, 1979.

\bibitem{marr80}
D.~Marr.
\newblock {\em Vision: A computational investigation into the human
  representation and processing of visual information}.
\newblock MIT press, 1980.

\bibitem{melas2023realfusion}
L.~Melas-Kyriazi, C.~Rupprecht, I.~Laina, and A.~Vedaldi.
\newblock Realfusion: 360 $\{$$\backslash$deg$\}$ reconstruction of any object
  from a single image.
\newblock {\em arXiv preprint arXiv:2302.10663}, 2023.

\bibitem{mildenhall2022nerf}
B.~Mildenhall, P.~Hedman, R.~Martin-Brualla, P.~P. Srinivasan, and J.~T.
  Barron.
\newblock Nerf in the dark: High dynamic range view synthesis from noisy raw
  images.
\newblock In {\em Proceedings of the IEEE/CVF Conference on Computer Vision and
  Pattern Recognition}, pages 16190--16199, 2022.

\bibitem{mildenhall2019local}
B.~Mildenhall, P.~P. Srinivasan, R.~Ortiz-Cayon, N.~K. Kalantari,
  R.~Ramamoorthi, R.~Ng, and A.~Kar.
\newblock Local light field fusion: Practical view synthesis with prescriptive
  sampling guidelines.
\newblock {\em ACM Transactions on Graphics (TOG)}, 38(4):1--14, 2019.

\bibitem{mildenhall2021nerf}
B.~Mildenhall, P.~P. Srinivasan, M.~Tancik, J.~T. Barron, R.~Ramamoorthi, and
  R.~Ng.
\newblock Nerf: Representing scenes as neural radiance fields for view
  synthesis.
\newblock {\em Communications of the ACM}, 65(1):99--106, 2021.

\bibitem{muller2022instant}
T.~M{\"u}ller, A.~Evans, C.~Schied, and A.~Keller.
\newblock Instant neural graphics primitives with a multiresolution hash
  encoding.
\newblock {\em ACM Transactions on Graphics (ToG)}, 41(4):1--15, 2022.

\bibitem{nelson1967dynamical}
E.~Nelson.
\newblock {\em Dynamical theories of Brownian motion}.
\newblock Princeton university press, 1967.

\bibitem{pietroski2018conjoining}
P.~M. Pietroski.
\newblock {\em Conjoining meanings: Semantics without truth values}.
\newblock Oxford University Press, 2018.

\bibitem{poole2022dreamfusion}
B.~Poole, A.~Jain, J.~T. Barron, and B.~Mildenhall.
\newblock Dreamfusion: Text-to-3d using 2d diffusion.
\newblock {\em arXiv preprint arXiv:2209.14988}, 2022.

\bibitem{ramesh2022hierarchical}
A.~Ramesh, P.~Dhariwal, A.~Nichol, C.~Chu, and M.~Chen.
\newblock Hierarchical text-conditional image generation with clip latents.
\newblock {\em arXiv preprint arXiv:2204.06125}, 2022.

\bibitem{reizenstein2021common}
J.~Reizenstein, R.~Shapovalov, P.~Henzler, L.~Sbordone, P.~Labatut, and
  D.~Novotny.
\newblock Common objects in 3d: Large-scale learning and evaluation of
  real-life 3d category reconstruction.
\newblock In {\em Proceedings of the IEEE/CVF International Conference on
  Computer Vision}, pages 10901--10911, 2021.

\bibitem{rombach2021highresolution}
R.~Rombach, A.~Blattmann, D.~Lorenz, P.~Esser, and B.~Ommer.
\newblock High-resolution image synthesis with latent diffusion models, 2021.

\bibitem{rombach2022high}
R.~Rombach, A.~Blattmann, D.~Lorenz, P.~Esser, and B.~Ommer.
\newblock High-resolution image synthesis with latent diffusion models.
\newblock In {\em Proceedings of the IEEE/CVF Conference on Computer Vision and
  Pattern Recognition}, pages 10684--10695, 2022.

\bibitem{russell}
B.~Russell.
\newblock {\em The Problems of Philosophy, Chapter 1: Appearance and Reality}.
\newblock OUP Oxford, 2001.

\bibitem{saharia2022photorealistic}
C.~Saharia, W.~Chan, S.~Saxena, L.~Li, J.~Whang, E.~L. Denton, K.~Ghasemipour,
  R.~Gontijo~Lopes, B.~Karagol~Ayan, T.~Salimans, et~al.
\newblock Photorealistic text-to-image diffusion models with deep language
  understanding.
\newblock {\em Advances in Neural Information Processing Systems},
  35:36479--36494, 2022.

\bibitem{inception_score_paper}
T.~Salimans, I.~Goodfellow, W.~Zaremba, V.~Cheung, A.~Radford, and X.~Chen.
\newblock Improved techniques for training gans, 2016.

\bibitem{schonberger2016structure}
J.~L. Schonberger and J.-M. Frahm.
\newblock Structure-from-motion revisited.
\newblock In {\em Proceedings of the IEEE conference on computer vision and
  pattern recognition}, pages 4104--4113, 2016.

\bibitem{soatto2023taming}
S.~Soatto, P.~Tabuada, P.~Chaudhari, and T.~Y. Liu.
\newblock Taming ai bots: Controllability of neural states in large language
  models.
\newblock {\em ArXiv}, 2023.

\bibitem{suhail2022light}
M.~Suhail, C.~Esteves, L.~Sigal, and A.~Makadia.
\newblock Light field neural rendering.
\newblock In {\em Proceedings of the IEEE/CVF Conference on Computer Vision and
  Pattern Recognition}, pages 8269--8279, 2022.

\bibitem{nerfstudio}
M.~Tancik, E.~Weber, E.~Ng, R.~Li, B.~Yi, J.~Kerr, T.~Wang, A.~Kristoffersen,
  J.~Austin, K.~Salahi, A.~Ahuja, D.~McAllister, and A.~Kanazawa.
\newblock Nerfstudio: A modular framework for neural radiance field
  development.
\newblock {\em arXiv preprint arXiv:2302.04264}, 2023.

\bibitem{trager2023linear}
M.~Trager, P.~Perera, L.~Zancato, A.~Achille, P.~Bhatia, B.~Xiang, and
  S.~Soatto.
\newblock Linear spaces of meanings: the compositional language of vlms.
\newblock {\em arXiv preprint arXiv:2302.14383}, 2023.

\bibitem{ulyanov2018deep}
D.~Ulyanov, A.~Vedaldi, and V.~Lempitsky.
\newblock Deep image prior.
\newblock In {\em Proceedings of the IEEE conference on computer vision and
  pattern recognition}, pages 9446--9454, 2018.

\bibitem{verbin2022ref}
D.~Verbin, P.~Hedman, B.~Mildenhall, T.~Zickler, J.~T. Barron, and P.~P.
  Srinivasan.
\newblock Ref-nerf: Structured view-dependent appearance for neural radiance
  fields.
\newblock In {\em 2022 IEEE/CVF Conference on Computer Vision and Pattern
  Recognition (CVPR)}, pages 5481--5490. IEEE, 2022.

\bibitem{wang2022score}
H.~Wang, X.~Du, J.~Li, R.~A. Yeh, and G.~Shakhnarovich.
\newblock Score jacobian chaining: Lifting pretrained 2d diffusion models for
  3d generation.
\newblock {\em arXiv preprint arXiv:2212.00774}, 2022.

\bibitem{wang2022pretraining}
T.~Wang, T.~Zhang, B.~Zhang, H.~Ouyang, D.~Chen, Q.~Chen, and F.~Wen.
\newblock Pretraining is all you need for image-to-image translation.
\newblock {\em arXiv preprint arXiv:2205.12952}, 2022.

\bibitem{wang2021real}
X.~Wang, L.~Xie, C.~Dong, and Y.~Shan.
\newblock Real-esrgan: Training real-world blind super-resolution with pure
  synthetic data.
\newblock In {\em Proceedings of the IEEE/CVF International Conference on
  Computer Vision}, pages 1905--1914, 2021.

\bibitem{wang2018esrgan}
X.~Wang, K.~Yu, S.~Wu, J.~Gu, Y.~Liu, C.~Dong, Y.~Qiao, and C.~Change~Loy.
\newblock Esrgan: Enhanced super-resolution generative adversarial networks.
\newblock In {\em Proceedings of the European conference on computer vision
  (ECCV) workshops}, pages 0--0, 2018.

\bibitem{ssim_paper}
Z.~Wang, A.~Bovik, H.~Sheikh, and E.~Simoncelli.
\newblock Image quality assessment: from error visibility to structural
  similarity.
\newblock {\em IEEE Transactions on Image Processing}, 13(4):600--612, 2004.

\bibitem{warburg2023nerfbusters}
F.~Warburg, E.~Weber, M.~Tancik, A.~Holynski, and A.~Kanazawa.
\newblock Nerfbusters: Removing ghostly artifacts from casually captured nerfs.
\newblock {\em arXiv preprint arXiv:2304.10532}, 2023.

\bibitem{wizadwongsa2021nex}
S.~Wizadwongsa, P.~Phongthawee, J.~Yenphraphai, and S.~Suwajanakorn.
\newblock Nex: Real-time view synthesis with neural basis expansion.
\newblock In {\em Proceedings of the IEEE/CVF Conference on Computer Vision and
  Pattern Recognition}, pages 8534--8543, 2021.

\bibitem{wynn2023diffusionerf}
J.~Wynn and D.~Turmukhambetov.
\newblock Diffusionerf: Regularizing neural radiance fields with denoising
  diffusion models.
\newblock {\em arXiv preprint arXiv:2302.12231}, 2023.

\bibitem{xu2022neurallift}
D.~Xu, Y.~Jiang, P.~Wang, Z.~Fan, Y.~Wang, and Z.~Wang.
\newblock Neurallift-360: Lifting an in-the-wild 2d photo to a 3d object with
  360 $\{$$\backslash$deg$\}$ views.
\newblock {\em arXiv preprint arXiv:2211.16431}, 2022.

\bibitem{yao2020blendedmvs}
Y.~Yao, Z.~Luo, S.~Li, J.~Zhang, Y.~Ren, L.~Zhou, T.~Fang, and L.~Quan.
\newblock Blendedmvs: A large-scale dataset for generalized multi-view stereo
  networks.
\newblock In {\em Proceedings of the IEEE/CVF Conference on Computer Vision and
  Pattern Recognition}, pages 1790--1799, 2020.

\bibitem{zhang2023adding}
L.~Zhang and M.~Agrawala.
\newblock Adding conditional control to text-to-image diffusion models, 2023.

\bibitem{lpips_paper}
R.~Zhang, P.~Isola, A.~A. Efros, E.~Shechtman, and O.~Wang.
\newblock The unreasonable effectiveness of deep features as a perceptual
  metric.
\newblock In {\em CVPR}, 2018.

\bibitem{zhou2022sparsefusion}
Z.~Zhou and S.~Tulsiani.
\newblock Sparsefusion: Distilling view-conditioned diffusion for 3d
  reconstruction.
\newblock {\em arXiv preprint arXiv:2212.00792}, 2022.

\bibitem{zoran2009scale}
D.~Zoran and Y.~Weiss.
\newblock Scale invariance and noise in natural images.
\newblock In {\em 2009 IEEE 12th International Conference on Computer Vision},
  pages 2209--2216. IEEE, 2009.

\end{thebibliography}

\appendix

\begin{figure}[h]
    \centering
    \includegraphics[width=0.99\textwidth]{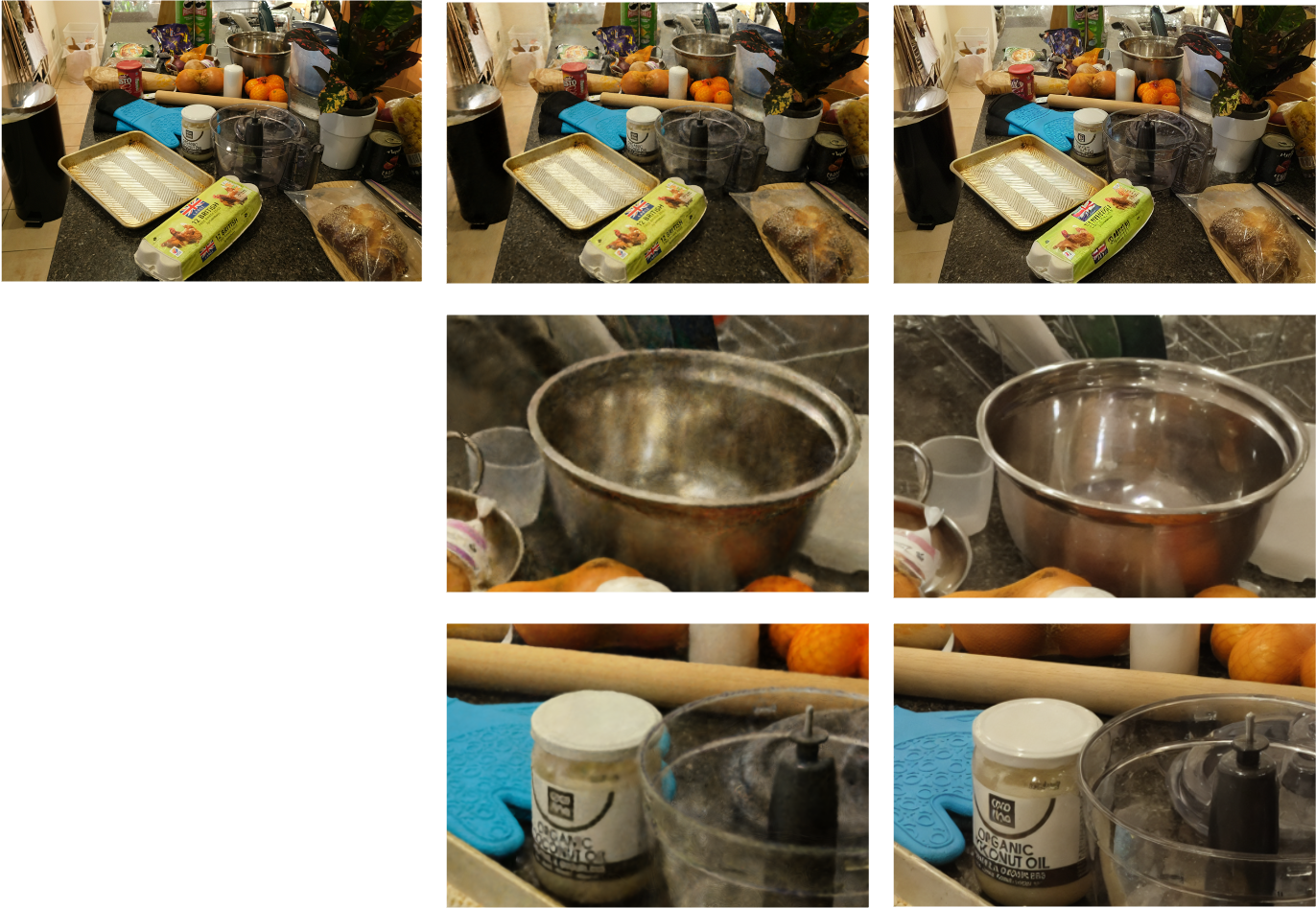}
    \caption{NeRFs cannot interpolate at finer resolution without visible artifacts, that betray the NeRF as not equivalent to other scenes, in the sense of Theorem \ref{thm:uniqueness-scene}, and therefore non viable as a representation of the physical scene.
    \textbf{Left to right}: ({\bf a}) ground-truth reference image, ({\bf b}) NeRF rendering, ({\bf c}) NeRF Diffusion output. The top row corresponds to test samples of the \texttt{counter} sequence in Mip-NeRF 360 dataset. In the second and third rows, we randomly sample zoomed-in camera poses around the object and obtain corresponding NeRF rendering (column {\bf b}) and NeRF Diffusion output (column {\bf c}).}
    \label{fig:supp-nerf-artifacts}
\end{figure}

\begin{figure}[h]
    \centering
    \includegraphics[width=\textwidth]{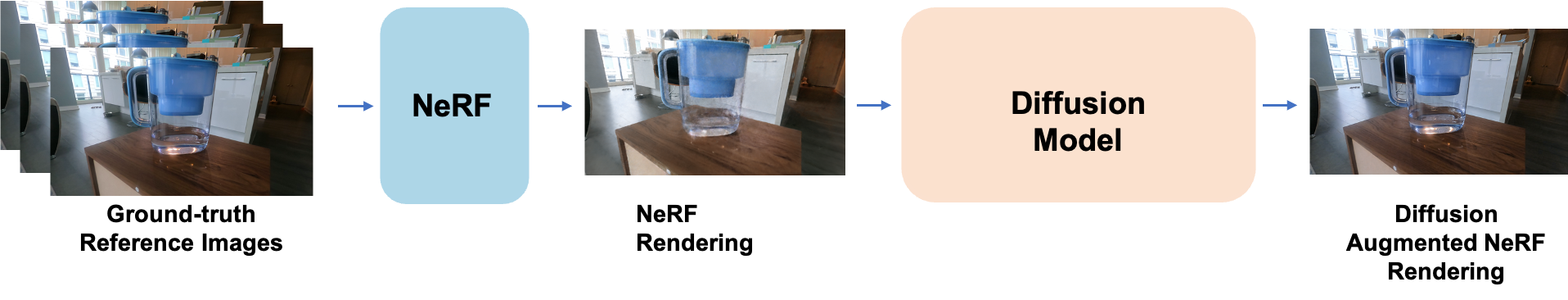}
    \caption{ A block diagram of our simple model. }
    \label{fig:block_diagram}
\end{figure}

\section{Q\&A}

In this section, we discuss some of the most controversial elements of our paper.

\begin{itemize}
    \item {\bf Q:} {\bf What does this paper have to do with physical scenes? There is no talk of surfaces, reflectance, motion, any of the aspects described as critical for interaction with physical space.}
    \item {\bf A:} Whatever properties we seek to infer about the scene, including geometry, photometry, dynamics, etc., can only be inferred if they are manifest in the data. Any representation that can be inferred from the data and is {\em maximally informative} (sufficient, per Definition \ref{defn:scene}), can be used to {\em infer any properties or attributes of the scene, even if they are not explicit in the representation.} 
    \item {\bf Q:} {\bf But there is no way to pinpoint 3D shape in the parameters of a NeRF or a Diffusion Model. Where is the shape?}
    \item {\bf A:} Shape information (not the shape itself) is implicit in the parameters of the trained model, and in the activations after feeding all the suitable conditioning inputs. The real question to ask, here, is {\em where is the shape in the ``real scene''?} Describing the geometry of a scene as a collection of piecewise smooth and multiply-connected surfaces, the standard in the so-called Marrian approach to Vision \cite{koenderink2011vision}, is a gross oversimplification that excludes co-dimension one structures (wires), diffuse structures (clouds), and even looking closely at some surfaces ({\em e.g.}, fabric, hair) it is not clear that there actually is a surface at all, depending on the scale of observation. So, the description of a scene as a configuration of surfaces, although useful in some applications such as robotic navigation, has no more ontological value than the parameters of a trained model, since the latter can be used to infer the former, but not vice-versa. Once restricted to a collection of reflective surfaces, an inferred scene cannot, for instance, be used to model visibility and illumination  artifacts such as transparency, translucency, and inter-reflections. In fact, these are considered ``outliers'' in traditional 3D scene modeling. They are, on the other hand, captured in the representations we have described. 
    \item {\bf Q:} {\bf You refer to extrapolation as hallucination, but these are two distinct phenomena! Hallucination makes up stuff that is not there in reality.}
    \item {\bf A:} So does extrapolation. Given a set of data, any form of extrapolation requires some kind  of prior, regularizer, or inductive bias, which in a trained model is embodied by the training set, the architecture, the training loss, and the optimization method. These enable one to fill in details not visible in one scene {\em using information from other scenes}. In standard ill-posed inverse problems, such extrapolation is fostered by generic regularizers ({\em e.g.}, minimal curvature, maximum sparsity, etc.), but neither the validity of the prior nor the inductive transfer are falsifiable because they are tasked with imputing information from scenes {\em other than the one in question}, which does not exist in the latter. 
    \item {\bf Q:} {\bf In what sense is this a Foundation Model? Foundation Models are supposed to be {\em homogeneous and task-agnostic} representations that can support any downstream tasks. NeRF Diffusions are neither homogeneous nor task-agnostic.}
    \item {\bf A:} Optimal representations are functions of the data that are minimal sufficient (known also as Sufficient Invariants \cite{achille2022binding}). This is a well-defined and defining criterion of a representation. Task-agnosticism, on the other hand, is misleading for the only representation that is truly task-agnostic is the data itself, or any lossless representation of it.  Even learning criteria branded as task-agnostic, such as contrastive learning and so-called ``self-supervised'' learning, are very much task-dependent: The task is implicit in the design choice of transformation, data augmentation, or surrogate loss. 
    \item {\bf Q:} {\bf Foundation Models are Transformer architectures. NeRFs and Diffusion Models are not Foundation Models.}
    \item {\bf A:} Transformers enjoy some desirable properties that make them suitable for use as Foundations, since they are Turing Complete, a characteristic needed to represent abstract concepts. However, they are not the only ones. So are RNNs and NeRF Diffusions, as we have shown here by proving that NeRF Diffusions can represent the physical scene, which is an abstract concept.
    \item {\bf Q:} {\bf You say that a NeRF cannot represent the scene because it is feed-forward, but an MLP enjoys the universal approximation properties of neural networks, so something is wrong.}
    \item {\bf A:} The Universal Approximation Theorem \cite{cybenko1989approximation} states that one can approximate a function {\em in a compact domain}. Scenes do not live in compact domains! If that was the case, we could discretize them and be done with it.
    \item {\bf Foundation Models are Large Language Models, how does this relate to LLMs?} Transformer architectures pre-trained as masked autoencoders or predictors, and then fine-tuned, possibly using reinforcement learning machinery, on completed sequences have certainly been successful in NLP. While vectorized tokens encode elements in a finite dictionary, there is nothing unique about them representing (sub-)words in a natural language. The same machinery can be used for visual data, consistent with all the derivations in this paper. A further relation is the LLMs, as foundational models for text, are trained to represent ``meanings'' which are equivalence classes of sentences \cite{soatto2023taming}. A visual foundation model should be trained to represent ``scenes,'' which as we saw are equivalence classes of images. The major different is that language admits a natural discretization, being born discrete. Images, on the other hand, cannot be quantized in a way that easily relates to the underlying scene, as the same object can appear smaller than a pixel or larger than the entire image depending on the relative pose to the viewer.
    \item {\bf The scene is something we evolve to interact with, this paper does not talk about interaction, multimodality, and all of that. Would any of these ideas survive if we expand the scope to a more realistic setting?}
    \item {\bf A:} In theory, our derivation pertains to any data modality, so it is general. In practice, however, there are modality-specific characteristics that we did not delve into here. For example, remote sensors such as vision behave differently from contact sensors such as touch, and localized modalities differently from diffuse ones such as smell. Most importantly, {\em active modalities} where the outcome of inference affect the data acquisition process are clearly essential for the development of cognitive abilities (plans do not have a central nervous system), which we do not explore here. 
    \item {\bf Q:} {\bf you say that the physical scene is an abstract entity, which is contradictory: If it is physical, how can it be abstract?} 
    \item {\bf A:} We have argued that the ``true'' scene, sometimes referred to as ``real'' or ``physical'' scene, cannot be known. So, we define ``physical scene'' as one that has some uniqueness properties, so that it can be consistent, if not objective, even if observed by multiple agents, each of which processes the data differently. Once so defined, the physical scene is a function of the data stored in the memory of a model, which implements a computable function, that embodies an abstract concept. In general, not just our definition, but any meaningful notion of scene that is inferred from data is an abstract concept, and attempts to define some sort of objectivity inevitably turn into tautologies, a point eloquently explained by Koenderink in his critique of the objective, or so-called Marrian,  account \cite{koenderink2011vision} as well as by Russell \cite{einstein-russell}.
    \item {\bf Q:} {\bf you also say that physical laws are abstract entities. How can that be?}
    \item {\bf A:} The laws of physics are a human construct, expressed in human language, that live inside the human brain and are communicated through abstract symbols. These symbols describe relations among quantities, such as $F = ma$, that are not ``real'' but rather abstract: There are no measurements of force, mass and acceleration that, if plugged into the above equation, would yield the equal sign. That equation, therefore, is not a relation among real measurable quantities, but rather abstractions that can be easily explained and communicated among humans. Trained models may develop their own inner language, as argued in \cite{trager2023linear}, and it is unclear whether that can be ``translated'' to human natural language so that a representation can ``explain'' its own version of laws of physics, but all this is beyond our scope here and discussed to some extent in \cite{achille2022binding}.
    \item {\bf Q: The way in which NeRFs are combined with Diffusion Models is sub-optimal.}
    \item {\bf A:} There are surely more sophisticated ways of combining NeRFs and Diffusion Models, but this is neither our goal nor the main contribution claimed. Our goal is to test whether even a simple concatenation of the two can suffice to represent a physical scene, in theory - as well as with a modicum of empirical validation. The contribution is not the validation, but the theoretical framing of the question.
    \item{\bf Q: The definition of scene is trivially satisfied by the images themselves. It is not clear how this representation is meaningful and how it may represent the actual physical scene.}
    \item {\bf A:} As we point out in Sect.~\ref{sec:presentations}, indeed a collection of images is a viable representation of a scene, since from them one can infer anything that can be inferred from the scene {\em given those images}. But this is not a physical scene. As pointed out in the rest of that section, physical scenes are {\em equivalence classes of scenes.} where the equivalence classes are defined by the {\em presentations} of the underlying physical scene, if it exists.
    \item{\bf Q: Why do you need to know the pose? Does the fact that the NeRF is built with posed images invalidate the theory?} 
    \item{\bf A:} In theory, pose is not needed to learn the plenoptic function, as it can be factored out by the model given sufficient training data. In practice, we already know that {\em a sparse attributed point cloud is minimal sufficient for localization} so long as (a) the attributes (a.k.a. local features descriptors) are sufficient to establish correspondence for a sufficiently large collection of pairs of views, and (b) the number of correspondences is sufficient to define a reference frame (which in turn defines pose) despite violation of the three conditions under which correspondence is possible, which is co-visibility (occlusion), Lambertian reflectance and constant illumination. So, we can simplify the training of the NeRF, by reducing sample complexity, simply bypassing the complex optimization involved in factoring out pose simply by estimating it and inputting it along with the images. If pose is inaccurate, the resulting error will generate a bias, which averages out across samples and is manifest in visible artifacts.
    \item{\bf Q: What if the scene  is not static? Is there an assumption that the scene is immutable across the sensor measurements?}
    \item{\bf A:} Yes, the scene is defined as what persists among different views. If a scene contains multiple moving ``objects'' (which we have not defined here as that entails some complexities), each object represents a scene of its own. If objects are simply-connected surfaces supporting Lambertian reflection, this can be done easily. If objects are translucent, transparent, or reflective, there are global dependencies that complicate the analysis, which we defer to future work. 
    \end{itemize}

\end{document}